%% file: main.tex
\documentclass[letterpaper, 10 pt, journal, twoside]{IEEEtran}

\usepackage{comment}
\usepackage{amsmath}

\usepackage{amsthm}

\usepackage{graphicx}
\usepackage{rotating}
\usepackage{float}
\usepackage{mathrsfs}
\usepackage{amssymb}
\usepackage{autobreak}
\usepackage{mathtools}
\usepackage{wrapfig}
\usepackage[dvipsnames]{xcolor}
\usepackage{bbm}
\usepackage{algorithm}
\usepackage{graphicx, subcaption}
\usepackage{makecell}
\usepackage{booktabs}

\usepackage[]{algpseudocode}

\usepackage{lipsum}
\usepackage{outlines}
\usepackage{blindtext}
\usepackage{multicol}
\usepackage{tikz}
\usetikzlibrary{shapes,backgrounds}
\usepackage{array}
\usepackage{wrapfig}
\usepackage{thmtools} 
\usepackage{thm-restate}

\usepackage{epsfig} 
\usepackage{amsmath} 
\usepackage{algorithm}
\usepackage[]{algpseudocode}

\DeclareMathOperator*{\argmax}{arg\,max}
\DeclareMathOperator*{\argmin}{arg\,min}

\input{macros.tex}

\makeatletter
\let\NAT@parse\undefined
\makeatother
\definecolor{lightblue}{rgb}{0.12,0.49,0.85}
\usepackage[colorlinks=true,linkcolor=black,citecolor=blue,urlcolor=blue]{hyperref}
\usepackage[capitalise]{cleveref}
\crefname{assumption}{Assumption}{Assumptions}
\begin{document}

\title{Performance-driven Constrained Optimal Auto-Tuner for MPC}

% \markboth{IEEE Robotics and Automation Letters. Preprint Version. Accepted February, 2025}
% {Gassol Puigjaner \MakeLowercase{\textit{et al.}}: Performance-driven Constrained Optimal Auto-Tuner for MPC} 
\author{Albert Gassol Puigjaner, Manish Prajapat, Andrea Carron, Andreas Krause, Melanie N. Zeilinger%
\thanks{All authors are with ETH Zurich. {\tt\small\{agassol, manishp, carrona, krausea, mzeilinger\}@ethz.ch}}
}

\maketitle

\begin{abstract}
A key challenge in tuning Model Predictive Control (\MPC) cost function parameters is to ensure that the system performance stays consistently above a certain threshold. To address this challenge, we propose a novel method, \autotuner, Constrained Optimal Auto-Tuner for \MPC. With every tuning iteration, \autotuner gathers performance data and learns by updating its posterior belief.  It explores the tuning parameters' domain towards optimistic parameters in a goal-directed fashion, which is key to its sample efficiency.  We theoretically analyze \autotuner, showing that it satisfies performance constraints with arbitrarily high probability at all times and provably converges to the optimum performance within finite time. Through comprehensive simulations and comparative analyses with a hardware platform, we demonstrate the effectiveness of \autotuner in comparison to classical Bayesian Optimization (BO) and other state-of-the-art methods. When applied to autonomous racing, our approach outperforms baselines in terms of constraint violations and cumulative regret over time.
\end{abstract}

\vspace{0.1cm}
\noindent {\small {\bf \autotuner Code}: \href{https://github.com/albertgassol1/coatmpc}{https://github.com/albertgassol1/coat\_mpc}}

\noindent {\small {\bf CRS Code}: \href{https://gitlab.ethz.ch/ics/crs}{https://gitlab.ethz.ch/ics/crs}}

\noindent {\small {\bf Video}: \href{https://youtu.be/Ep_BX3BDaeU?si=ShPcvWB_I8xCGg9T}{https://youtu.be/Ep\_BX3BDaeU?si=ShPcvWB\_I8xCGg9T}}

\setlength{\textfloatsep}{5pt}
\input{1-introduction}
% Related work
\input{2-related_work}
% % Problem Statement
\input{3-background}
% % Algorithm
\input{4-algorithm}

% % Theoretical results
\input{5-theoretical_results}
% % Experimental results
\input{6-experimental_results}
% % Conclusion
\input{7-conclusions}

\appendix
\input{A2-sample-complexity}

\bibliographystyle{IEEEtran}
\bibliography{IEEEabrv, bibligraphy}

\end{document}

%% file: macros.tex
\newcommand{\betat}[1][]{\beta_{#1}}
\newcommand{\dyn}{\boldsymbol{f}}
\newcommand{\sempc}{sage}

\newcommand{\btheta}{\boldsymbol{\theta}}
\newcommand{\n}{n}
\newcommand{\m}{m}
\newcommand{\bhtheta}{\boldsymbol{\hat{\theta}}}

\newcommand{\autotuner}{\textsc{\small{COAt-MPC}}\xspace}
\newcommand{\safeopt}{\textsc{\small{SafeOpt}}\xspace}
\newcommand{\gpucb}{\textsc{\small{GP-UCB}}\xspace}
\newcommand{\wml}{\textsc{\small{WML}}\xspace}
\newcommand{\crbo}{\textsc{\small{CRBO}}\xspace}
\newcommand{\eic}{\textsc{\small{EI\textsubscript{C}}}\xspace}

\newcommand{\MPC}{\textsc{\small{MPC}}\xspace}

\newcommand{\goose}{\textsc{\small{GoOSE}}\xspace}

\newcommand{\ucb}{\textsc{\small{UCB}}\xspace}

\newcommand{\mypar}[1]{\noindent\textbf{#1}.}

\newcommand*{\defeq}{\mathrel{\vcenter{\baselineskip0.5ex \lineskiplimit0pt
                     \hbox{\scriptsize.}\hbox{\scriptsize.}}}%
                     =}

\newcommand{\bx}{\boldsymbol{x}}

\newcommand{\R}{\mathbb{R}}

\newtheorem{lemma}{Lemma}
\newtheorem{assumption}{Assumption}

\newtheorem{theorem*}{Theorem}
\newtheorem{corollary}{Corollary}

\newcommand{\Domain}{D}

\newcommand{\constrain}{q}

\newcommand{\noise}{\eta}

\newcommand{\OptiOperReach}[2][]{O^{#1}_n(#2)}
\newcommand{\optiOper}[2][]{o^{#1}_{#2}}

\newcommand{\tilOptiOper}[2][]{\tilde{O}^{#1}_n(#2)}

\newcommand{\PessiOperReach}[2][]{P^{#1}_n(#2)}
\newcommand{\pessiOper}[2][]{p^{#1}_{#2}}
\newcommand{\tilPessiOper}[2][]{\tilde{P}^{#1}_n(#2)}

\newcommand{\reachOper}[2][]{r^{#1}_{#2}}
\newcommand{\tilReachOper}[2][]{\tilde{R}^{#1}(#2)}
\newcommand{\ReachOperReach}[2][]{R^{#1}(#2)}

\newcommand{\ubconst}[1][]{u_{#1}}
\newcommand{\lbconst}[1][]{l_{#1}}

\newcommand{\epsconst}{\epsilon}

\newcommand{\noiseconst}{\sigma^{-2}_{\constrain}}

\newcommand{\betaconst}[1][]{\beta_{#1}}

\newcommand{\pessiSet}[2][]{\mathcal{S}_{#2}^{ p #1}}

\newcommand{\optiSet}[2][]{\mathcal{S}_{#2}^{ o, \epsconst #1}}
\newcommand{\constSet}[2][]{\mathcal{S}_{#2}^{ \constrain #1}}

\newcommand{\sumMaxwidth}[2][]{w^{#1}_{#2}}

\newcommand{\safeset}{\mathcal{S}}

\newcommand{\bu}{\boldsymbol{u}}
\newcommand{\bbigq}{\boldsymbol{Q}}
\newcommand{\bbigr}{\boldsymbol{R}}
\newcommand{\bxhat}{\boldsymbol{\hat{x}}}

\newcommand{\xset}{\mathcal{X}}
\newcommand{\uset}{\mathcal{U}}

\newcommand{\noisew}{\sigma_{\noise}}
\newcommand{\bk}{\boldsymbol{k}}
\newcommand{\bbigk}{\boldsymbol{K}}
\newcommand{\identity}{\boldsymbol{I}}
\newcommand{\by}{\boldsymbol{y}}

\newcommand{\seed}{\mathcal{S}_0}
\newcommand{\goal}{\btheta^g}

%% file: 1-introduction.tex
% !TeX spellcheck = en_US
% !TeX encoding = UTF-8
% !TeX root = ../main.tex

\section{Introduction}
\label{chp:Introduction}

\looseness -1 Model Predictive Control (\MPC) is a prominent optimization-based control framework that can handle constraints and optimize system performance by predicting the system's future behavior. \MPC is widely used in many robotic applications such as autonomous driving~\cite{Kabzan2019AMZDT}, four-legged robots~\cite{9981945}, and bipedal robots~\cite{Kuindersma2015OptimizationbasedLP}. While \MPC is a successful optimal control technique, one of the significant challenges in its implementation is tuning the cost function parameters. Designing a cost function that balances competing objectives is a non-trivial task that requires significant trial and error. Moreover, the cost function parameters often depend on the specific environment and system dynamics, making it difficult to design a single set of parameters that can perform well in all scenarios. Usually, the task of fine-tuning cost function parameters involves heuristic methods and demands expert knowledge, leading to a significant number of costly and time-consuming experimental iterations. 

\looseness -1 In most applications, we tune to maximize some performance function, e.g., while tuning for racing, we optimize the lap time. However, these performance functions are often \emph{a-priori} unknown and need to be learned through data.  Naively, to find the optimal parameters, one may try out a large set of parameters in a grid search approach.  Apart from inefficiency caused by executing a large number of parameters, many of those parameters may lead the system to halt, i.e., low performance, resulting in a wasteful evaluation. For example, in tuning an \MPC for autonomous racing, it is undesirable to use parameters that make the car move extremely slow, or even stop before finishing a lap. This motivates a tuning process ensuring system performance above a threshold.

\setlength{\textfloatsep}{5pt}
\begin{figure}[t]
        \centering
        \includegraphics[width=\columnwidth]{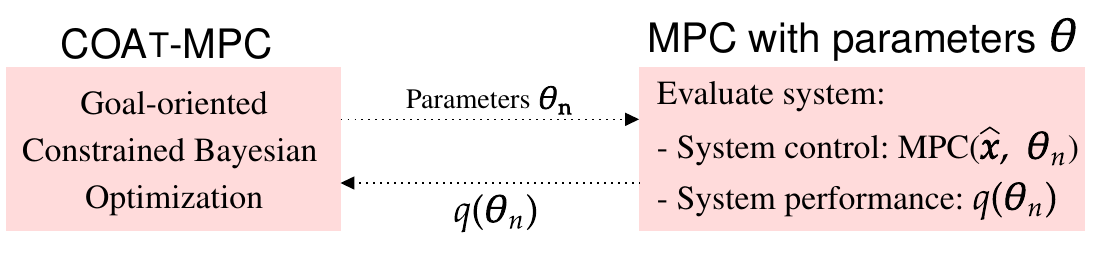}
        \caption{\looseness -1 \autotuner overview. \autotuner proposes a set of cost function weights $\btheta_\n$, which are evaluated on the system. It gets a performance function sample, which is used to update the posterior belief and acquire a new set of cost function weights. The process is repeated until convergence.}
        \label{fig:cbo_diagram}
\end{figure}

\looseness -1 To tackle these challenges, we propose a novel algorithm: \autotuner, Constrained Optimal Auto-Tuner for \MPC. \autotuner explores the space of MPC cost function parameters and builds a belief about the \emph{a-priori} unknown performance function through data, utilizing tools from Gaussian processes \cite{Rasmussen2003GaussianPF}. \autotuner incorporates safe exploration ideas from \cite{safeopt_2,goose,prajapat2024safe} and recursively recommends sufficiently informative parameters that ensure exploration while satisfying the performance constraint.  We establish convergence guarantees to the optimal tuning parameters in a finite number of samples while ensuring performance constraint satisfaction with an arbitrarily high probability. For finite time convergence, we present a sample complexity bound by leveraging the bound from~\cite{prajapat2024safe} for continuous domains and demonstrate its applicability with discrete set operators required for discrete domains. In particular, our sample complexity result removes an explicit dependence on the discretization step size and thus significantly improves prior safe exploration results in discrete domains~\cite{goose,safeopt_15,prajapat2022near}.

\looseness -1Finally, we demonstrate the effectiveness of \autotuner in the application of autonomous racing. We tune a Model Predictive Contouring Control~\cite{Liniger2015OptimizationbasedAR} formulation with the objective of optimizing the lap time while avoiding undesirable effects such as halting. Our evaluation includes a comprehensive analysis in both simulation and in experiments on a 1:28 scale RC racecar~\cite{Carron2022ChronosAC}. We present a comparative analysis against competitive baselines. The results demonstrate that our approach outperforms other methods in terms of the number of constraint violations and yields an improved cumulative regret. \vspace{-1em}

%% file: 2-related_work.tex
% !TeX spellcheck = en_US
% !TeX encoding = UTF-8
% !TeX root = ../main.tex

\section{Related works}
\label{chp:Related_works}

\looseness -1 Controller tuning in robotics has been an active area of research. Recently, data-driven methods aimed at learning the relationship between system parameters and a desired metric have emerged as promising solutions for automatic tuning.  Methods such as the Metropolis-Hastings algorithm~\cite{9696365} and Policy Search methods~\cite{Romero2022WeightedML}, have demonstrated state-of-the-art results in model-based agile flight control.  

Bayesian Optimization (BO)~\cite {Srinivas2009GaussianPO, Bo_tutorial} has been particularly successful due to its ability to optimize the objective function using a limited number of samples~\cite{Mockus1989BayesianAT}. It utilizes a probabilistic model to represent the unknown objective function, which is updated as new data is acquired. Even though any probabilistic model can be used, a popular choice to model unknown functions is a Gaussian Process (GP)~\cite{Rasmussen2003GaussianPF}, e.g., in BO ~\cite{safeopt_15,Frhlich2021ModelLA,pmlr-v144-frohlich21a,eriksson2019scalable,goose,menn2024lipschitz}, MPC~\cite{prajapat2024safe} or experiment design~\cite{prajapat2023submodular}. In the context of MPC parameter tuning, BO can be used to find the optimal cost function parameters for a given platform and environment~\cite{search_algos} by selecting an adequate objective function, i.e., laptime in the case of MPC tuning for racing applications. Additionally, subsequent works have introduced contextual information from the environment or system~\cite{Frhlich2021ModelLA}, considered the confidence of the probabilistic model to enhance the cautiousness and convergence rate of BO~\cite{pmlr-v144-frohlich21a}, and combined BO with trust region optimization~\cite{eriksson2019scalable}. However, it is worth noting that these BO methods do not take into account constraints on the objective function, which can lead to parameters that produce very poor performance due to their unbounded exploration. 

Several approaches have emerged to incorporate constraints into BO. One such approach involves utilizing a variant of the Expected Improvement (EI) function, referred to as the Constrained Expected Improvement (EI\textsubscript{C})~\cite{Gardner2014BayesianOW, Wilson2018MaximizingAF,SOROURIFAR2021243}. This approach involves modeling the constraint function with a prior distribution and incorporating a probability of violation into the acquisition function. However, none of these methods provide a theoretical guarantee of constraint satisfaction.

In the literature of constrained Bayesian Optimization, \textsc{Safe-Opt}~\cite{safeopt_15, safeopt_2, safeopt_3} is introduced as a method that aims to provide high-probability guarantees of constraint satisfaction. The algorithm leverages the regularity assumption on the objective function and the Lipschitz continuity to identify a set of parameters where the constraints on the underlying objective function are unlikely to be violated. Even though \textsc{Safe-Opt} has been proven to guarantee safety, it tends to explore the complete safe parameter region, leading to sample inefficiency in relation to the optimization task. A concurrent work~\cite{menn2024lipschitz} has explored constrained optimization for controller tuning, such as PID controllers in automotive applications, using a Lipschitz-only assumption approach.

In order to tackle sample inefficiency in safe exploration, the authors of~\cite{goose} propose \goose, a goal-oriented safe exploration algorithm for any interactive machine learning method. \goose leverages the regularity assumption on the constraint function to define over- and under-approximations of the safe set. A goal within the over-approximated set is defined at each iteration with the purpose of steering the recommendations of \goose towards the goal while ensuring safety. However, the sample complexity analysis of \goose is explicitly dependent on its discretization step size, leading to poor scalability  (in terms of finer discretization).

\looseness -1 Previous works have several limitations: they either do not incorporate constraints in the optimization~\cite{Srinivas2009GaussianPO,search_algos,Frhlich2021ModelLA,pmlr-v144-frohlich21a,eriksson2019scalable}, lack theoretical guarantees on constraint satisfaction~\cite{Gardner2014BayesianOW,Wilson2018MaximizingAF,SOROURIFAR2021243}, suffer from sample inefficiency~\cite{safeopt_15,safeopt_2,safeopt_3}, or provide poorly scalable theoretical sample complexity bounds~\cite{goose}. In this paper, we present a sample-efficient algorithm that satisfies performance constraints and offers scalable theoretical guarantees for \MPC cost function tuning.

%% file: 3-background.tex
% !TeX spellcheck = en_US
% !TeX encoding = UTF-8
% !TeX root = ../main.tex
\section{Problem statement} 
\label{chp:Problem_statement}
\looseness -1 We consider a non-linear dynamic system controlled using an \MPC with cost function parameters $\btheta \in \R^{N_{\theta}}$.

\begin{equation}
\begin{aligned}
     \min_{\bu_{\mathrm{0:N}}} \quad &  \sum_{i=0}^{N}{l(\bx_i, \bu_i, \btheta)} \\
    \textrm{s.t.} \quad & \bx_0 = \bxhat(t) , \, \bx_{i+1} = \dyn(\bx_i, \bu_i)\\
    &\bx_i \in \xset , \, \bu_i \in \uset, \forall i=0,\cdots, N\, ,
\label{eq:MPCFormulation}
\end{aligned}
\end{equation}

\noindent where \mbox{$\bx_i \in \R^{N_x}$} is the system state with dimension \mbox{$N_x$}, \mbox{$\bu_i \in \R^{N_u}$} is the control input with dimension $N_u$, \mbox{$\dyn(\bx_i, \bu_i):~\R^{N_x}\times \R^{N_u} \to \R^{N_x}$} denotes the system dynamics, \mbox{$l(\bx_i, \bu_i, \btheta):~\R^{N_x}\times \R^{N_u} \times \R^{ N_{\theta}} \to \R$} is the cost function and \mbox{$\bxhat(t) \in \R^{N_x}$} is the system state at time $t$.

We define a performance function \mbox{$\constrain\,:\,\Domain \rightarrow \R$}, where \mbox{$\Domain \subseteq \R^{N_{\theta}}$} is a finite domain of cost function parameters, that measures the performance of a given set of tuning parameters. The function $q(\btheta)$ is \emph{a-priori} unknown and needs to be learned with data. To learn the performance function, at any iteration $\n$, one can control the system with an \MPC using any parameter \mbox{$\btheta_\n \in \Domain$} and obtain a noisy observation of \mbox{$\constrain(\btheta_\n)$}. We examine the problem of finding the parameters that maximize $\constrain$ while ensuring that the performance is above a user-specified threshold \mbox{$\tau \in \R$} (e.g., an arbitrary upper bound lap time for racing applications) in all iterations, i.e.,
\begin{align}
    q(\btheta_\n) \geq \tau, \forall \n \geq 1.
    \label{eq:constraint}
\end{align}
 Ideally, we do not want to execute all parameters, but only those that are essential to guarantee convergence to optimal parameters while always satisfying the constraint. 

We next make an initialization assumption which is crucial to start the tuning process.
\begin{assumption}[Initial seed]
\label{assump:safe_seed}
    An initial set of parameters \mbox{$\safeset_0 \subseteq \Domain$} that satisfy the performance constraint is known, i.e., \mbox{$\forall \btheta\in \safeset_0, \constrain(\btheta)\geq \tau$}. 
\end{assumption} 
For a suitable $\tau$, this assumption can be satisfied by employing an MPC capable of controlling the system to obtain measurements of the performance function $\constrain$.  For instance, in autonomous racing, employing parameters of an MPC (which need not be optimized) capable of leading the car to complete the lap will meet the criteria outlined in~\cref{assump:safe_seed}.

Using ~\Cref{assump:safe_seed}, we construct a reachable set of cost function parameters, $\constSet[,\epsilon]{}$, which contains all the parameters that can be reached starting from the initial set $\safeset_0$, while always satisfying the performance constraint up to a statistical confidence of $\epsilon$-margin, i.e, \mbox{$\constrain(\btheta) -\epsilon \geq \tau$} (see~\cref{sec:goal-oriented-se} for details on how to construct this set).

\looseness -1 \mypar{\autotuner objective} Given the noisy measurements of the performance function, the best any algorithm can guarantee is convergence to parameters $\goal$ satisfying
\begin{equation}
    \begin{aligned}
    \constrain(\goal) \geq \max_{\btheta \in \constSet[,\epsconst]{}} \quad & \constrain(\btheta) - \epsilon
    \label{eq:HighLevelCBO}
    \end{aligned}
\end{equation}
in a finite number of tuning iterations while ensuring~\cref{eq:constraint}, where $\epsilon$ controls the tolerance between the converged and optimal parameters.

\section{Background}
\label{chp:background}
\looseness -1 In this section, we first introduce Gaussian Processes in~\cref{sec:GP}, which are used to model the unknown function $\constrain$, and utilize them to explain concepts of safe exploration relevant for \autotuner in~\cref{sec:goal-oriented-se}.

\subsection{Gaussian processes} \label{sec:GP}
\looseness -1

The performance function $\constrain$ is \emph{a-priori} unknown. Therefore, to explore the parameter space while satisfying the performance constraint, we need a mechanism that ensures that knowing about $\constrain$ at a certain $\btheta$ provides us with some information about the neighboring region. To this end, we make the following regularity assumptions on the performance function $\constrain$.
\begin{assumption} 
    The domain $\Domain$ is endowed with a positive definite kernel $ k_{\constrain}(\cdot, \cdot)$, and $\constrain$ has a bounded norm in the associated Reproducing Kernel Hilbert Space (RKHS)~\cite{kernels}, $||\constrain||_k \leq B_{\constrain} < \infty$.
\label{assump:q_RKHS}
\end{assumption}
This assumption allows us to model the performance function $\constrain$ using a Gaussian Process~\cite{Rasmussen2003GaussianPF}. GPs are probability distributions over a class of continuous smooth functions. GPs are characterized by a mean \mbox{$\mu:\R^{N_{\theta}} \to \R$} and a kernel function \mbox{$k:\R^{N_{\theta}}\times\R^{N_{\theta}} \to \R$}, which captures the notion of similarity between data points. Without loss of generality, we normalize such that \mbox{$k(\btheta,\btheta') \leq 1, \forall \btheta,\btheta' \in \R^{N_{\theta}}$}. Given a set of $\n$ noisy samples collected at \mbox{$A_{\n} = \{\btheta_i\}_{i=1}^n$}, perturbed by \mbox{$\eta_{\n}$} conditionally $\sigma_\eta$-sub-Gaussian noise, given by \mbox{$\by_{\n} = [\constrain(\btheta_1) + \eta_1, \hdots, \constrain(\btheta_{\n}) + \eta_{\n}]^\top$}, we can compute the posterior over $\constrain$ in closed form using,
\begin{equation}    
\begin{aligned}
    \mu_{\n}(\btheta) &= \bk_{\n}^{\top}(\btheta)(\bbigk_{\n} + \identity_{\n} \noisew^2)^{-1}\by_{\n},
      \\ 
    k_{\n}(\btheta, \btheta') &= k_{\n}(\btheta, \btheta') - \bk_{\n}^\top(\btheta)(\bbigk_{\n} + \identity_{\n} \noisew^2)^{-1}\bk_{\n}(\btheta'), \\
    \sigma_{\n}(\btheta) &= \sqrt{k_{\n}(\btheta,\btheta)}, \label{eq:posterior_update}
\end{aligned}
\end{equation}

\noindent where the covariance matrix $\bbigk_{\n}$ is defined as  \mbox{$\bbigk_{\n}(i,j) = k_{\n}(\btheta_i, \btheta_j), \,  i, j \in \{1, \hdots, \n\}$}, and \mbox{$\bk_{\n}(\btheta) = [k_{\n}(\btheta_1, \btheta), \hdots, k_{\n}(\btheta_{\n}, \btheta)]^\top$} and \mbox{$\sigma_{\n}:\R^{N_{\theta}} \to \R$} denotes the predictive variance. 

\cref{assump:q_RKHS} is a typical assumption in prior works that use GPs to model unknown functions~\cite{Frhlich2021ModelLA, goose, safeopt_15}.
We consider that the performance function $\constrain$ is $L$-Lipschitz continuous with respect to some metric $d$ on $\Domain$, e.g., the Euclidean metric. This is automatically satisfied when using common isotropic kernels, such as the Mátern and Gaussian kernels. Additionally, we define the maximum \textit{information capacity} \mbox{$\gamma_{\n} \defeq \sup_{ A \subseteq \Domain \, : \, |A|\leq n} I(\by_A; \constrain_A)$} associated with the kernel $k$, where $I(\by_A; \constrain_A)$ denotes the mutual information between $\constrain$ evaluated at locations in the set $A$ and the noisy samples $\by_A$ collected at $A$~\cite{Srinivas2009GaussianPO}. This definition lets us build upon the finite time convergence of \autotuner (\cref{chp:Theoretical_results}).

\subsection{Safe exploration}
\label{sec:goal-oriented-se}

\looseness -1 In this section, we introduce the necessary tools from prior works~\cite{goose,safeopt_15} required to explore the domain of cost function weights efficiently while ensuring~\cref{eq:constraint}. 

\mypar{Optimistic, pessimistic and reachable sets}  Utilizing the GP posterior~\cref{eq:posterior_update}, we construct intersecting lower and upper confidence bounds on $\constrain$ at each iteration $\n\geq1$ as:
\begin{equation}
\begin{aligned}
    l_{\n}(\btheta) &:= \max{(l_{\n-1}(\btheta), \mu_{\n-1}(\btheta) - \sqrt{\beta_{\n}}\sigma_{\n-1}(\btheta))}, \\
    u_{\n}(\btheta) &:= \min{(u_{\n-1}(\btheta), \mu_{\n-1}(\btheta) + \sqrt{\beta_{\n}}\sigma_{\n-1}(\btheta))},
    \label{eq:conf_bounds}
\end{aligned}
\end{equation}

\noindent initialized with \mbox{$l_{0}(\btheta) = \mu_{0}(\btheta) - \sqrt{\beta_{1}}\sigma_{0}(\btheta)$} and \mbox{$u_{0}(\btheta) = \mu_{0}(\btheta) + \sqrt{\beta_{1}}\sigma_{0}(\btheta)$}, where $\beta_n$ represents an appropriate scaling factor specified in~\cref{cor:beta}. Note that $l_{\n}(\cdot)$ is non-decreasing and $u_{\n}(\cdot)$ is non-increasing in $\n$, i.e., 
$$l_{\n+1}(\btheta) \geq l_{\n}(\btheta), u_{\n+1}(\btheta) \leq u_{\n}(\btheta), \forall \btheta \in \Domain,$$
directly by construction using intersecting confidence bounds in~\cref{eq:conf_bounds}. Using this and the GPs error bounds from Theorem 2 of~\cite{pmlr-v70-chowdhury17a}, we get the following corollary~\cite{safeopt_3}:

\begin{corollary}[Theorem 2~\cite{pmlr-v70-chowdhury17a}]
\label{cor:beta}
Let~\cref{assump:q_RKHS} hold. If \mbox{$\sqrt{\beta_{\n}} = B + 4 \sigma \sqrt{\gamma_{\n} + 1+ \ln(1/\delta)}$}, it holds that \mbox{$l_{\n}(\btheta) \leq \constrain(\btheta) \leq u_{\n}(\btheta), \forall\btheta\in\R^{N_{\theta}}$} with probability at least $1-\delta$. 
\end{corollary}

Throughout this work, we use $\sqrt{\betaconst[\n]}$ from ~\cref{cor:beta}. Similar to \goose~\cite{goose}, we define a one-step reachability operator exploiting the $L$-Lipschitz continuity of $\constrain$, and build pessimistic and optimistic constraint operators over it using the derived lower and upper confidence bounds.
\begin{align*}
   \reachOper[\epsilon]{}(\safeset) &= \{\btheta \in \Domain \,|\, \exists \btheta' \in \safeset : \constrain(\btheta') - Ld(\btheta, \btheta') - \epsilon \geq \tau \}\\
    \pessiOper[]{\n}(\safeset) &= \{ \btheta \in \Domain \,|\, \exists \btheta' \in \safeset : l_{\n}(\btheta') - Ld(\btheta, \btheta') \geq \tau \} \\
   \optiOper[\epsilon]{\n}(\safeset) &= \{ \btheta \in \Domain \,|\, \exists \btheta' \in \safeset : u_{\n}(\btheta') - Ld(\btheta, \btheta') - \epsilon \geq \tau \}.  \vspace{-1.5em}
\end{align*}\vspace{-1.5em}
\begin{figure}[t]
    \vspace{.2cm}
        \centering
        \includegraphics[width=1\columnwidth]{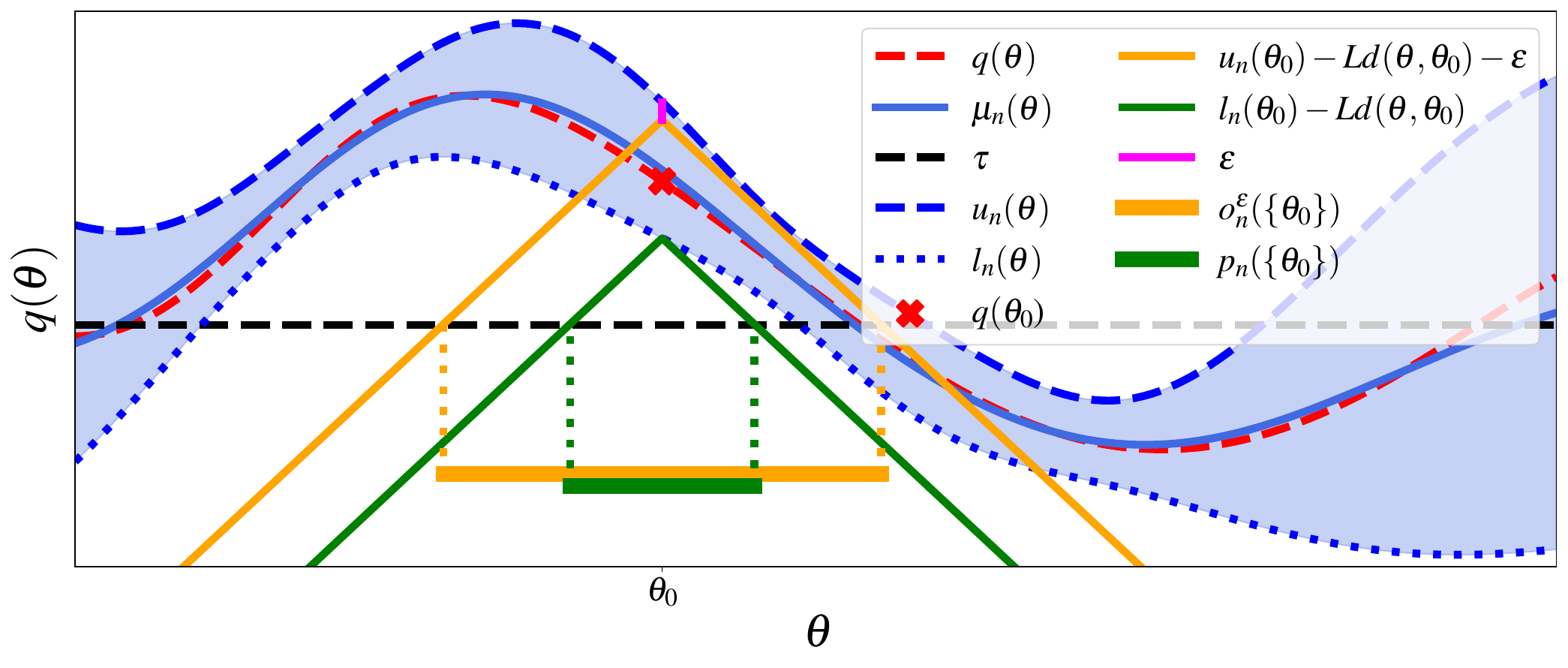}
        \caption{Pessimistic and optimistic operators evaluated at~$\theta_0$ (figure from~\cite{goose,prajapat2024safe}). This figure demonstrates how the pessimistic and optimistic sets are computed when only evaluated in a single point $\theta_0$. The operators make use of the GP upper and lower confidence bounds, as well as the $L$-Lipschitz continuity. In this example, $d(\theta, \theta_0)$ is the Euclidean distance function, where $\theta_0$ is fixed. Additionally, $\tau$ is set arbitrarily.}
        \label{fig:lipschitz}
            \vspace{-.2cm}
\end{figure}

\begin{figure*}[htp]
    \vspace{.2cm}
    \begin{subfigure}{.67\columnwidth}
        \centering
        \includegraphics[width=1\linewidth]{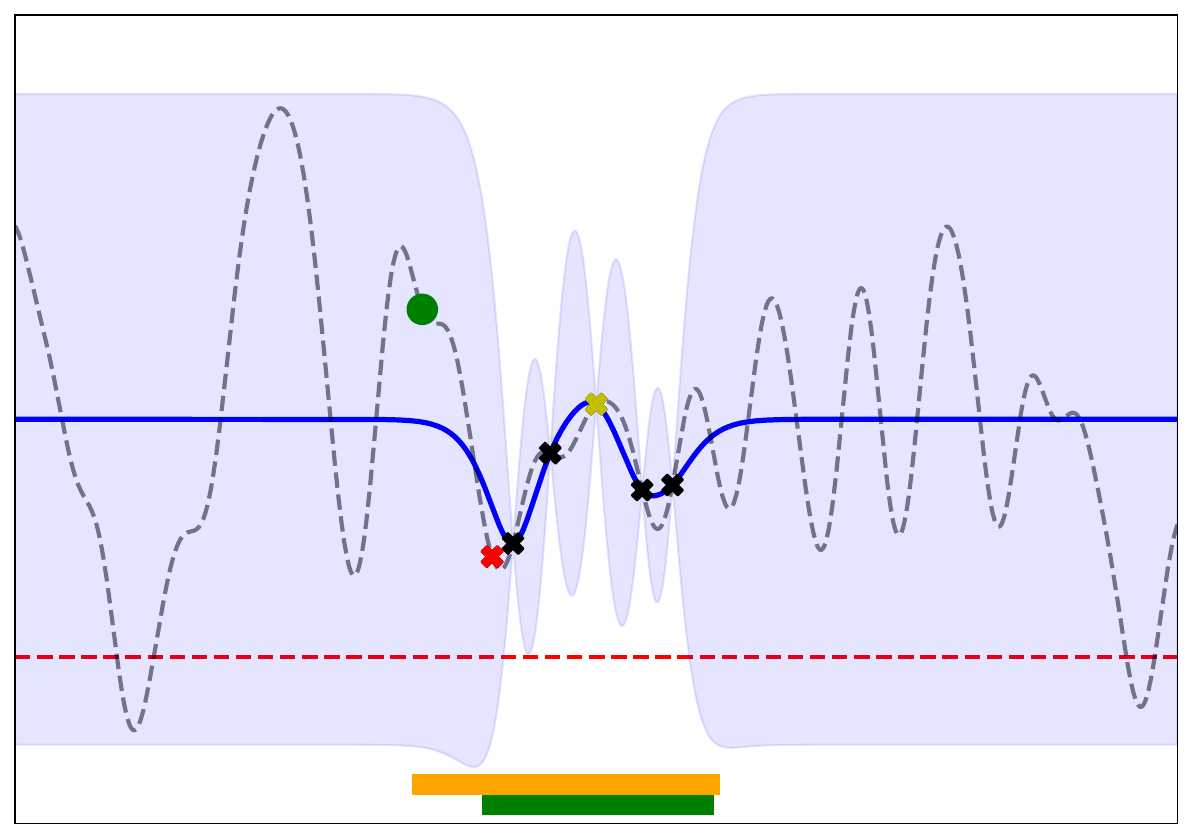}
        \caption{\autotuner at $\n=5$}
        \label{fig:sopt_1}
    \end{subfigure}
    \begin{subfigure}{.67\columnwidth}
        \centering
        \includegraphics[width=1\linewidth]{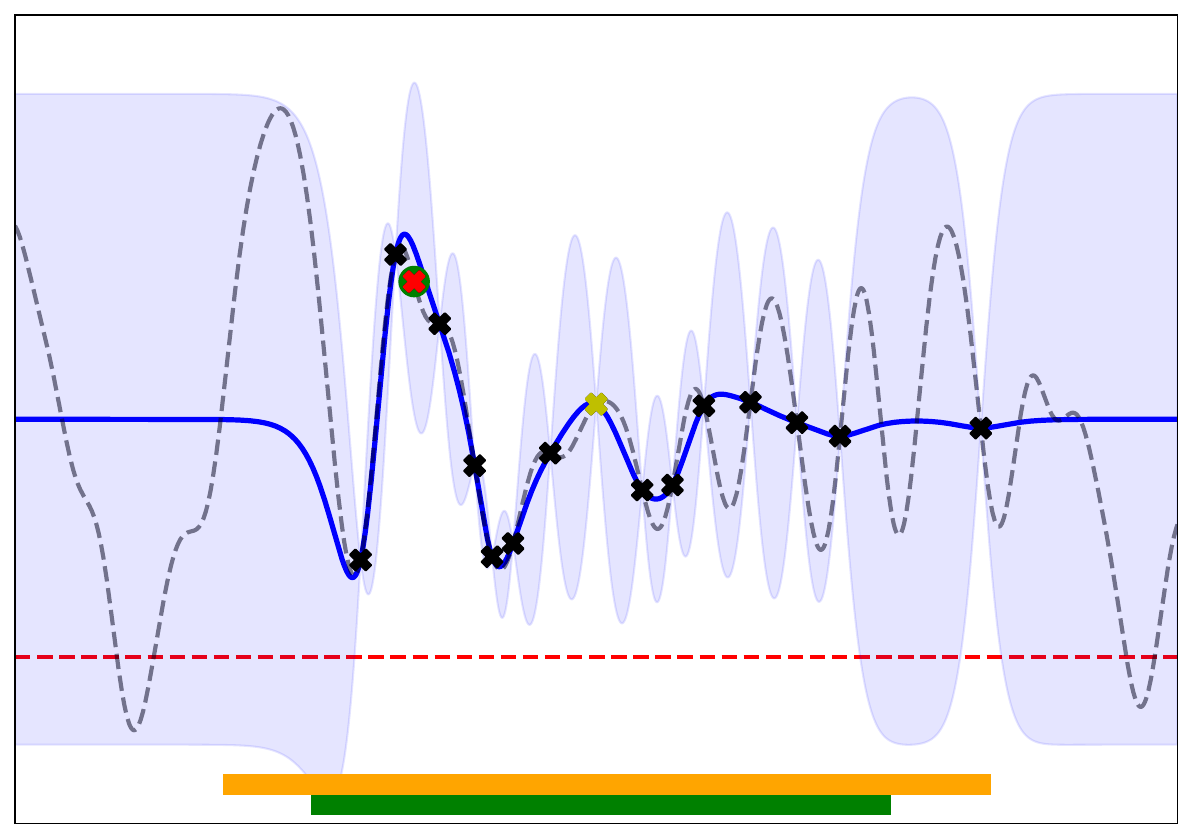}
        \caption{\autotuner at $\n=15$}
        \label{fig:crs_car}
    \end{subfigure}
    \begin{subfigure}{.67\columnwidth}
        \centering
        \includegraphics[width=1\linewidth]{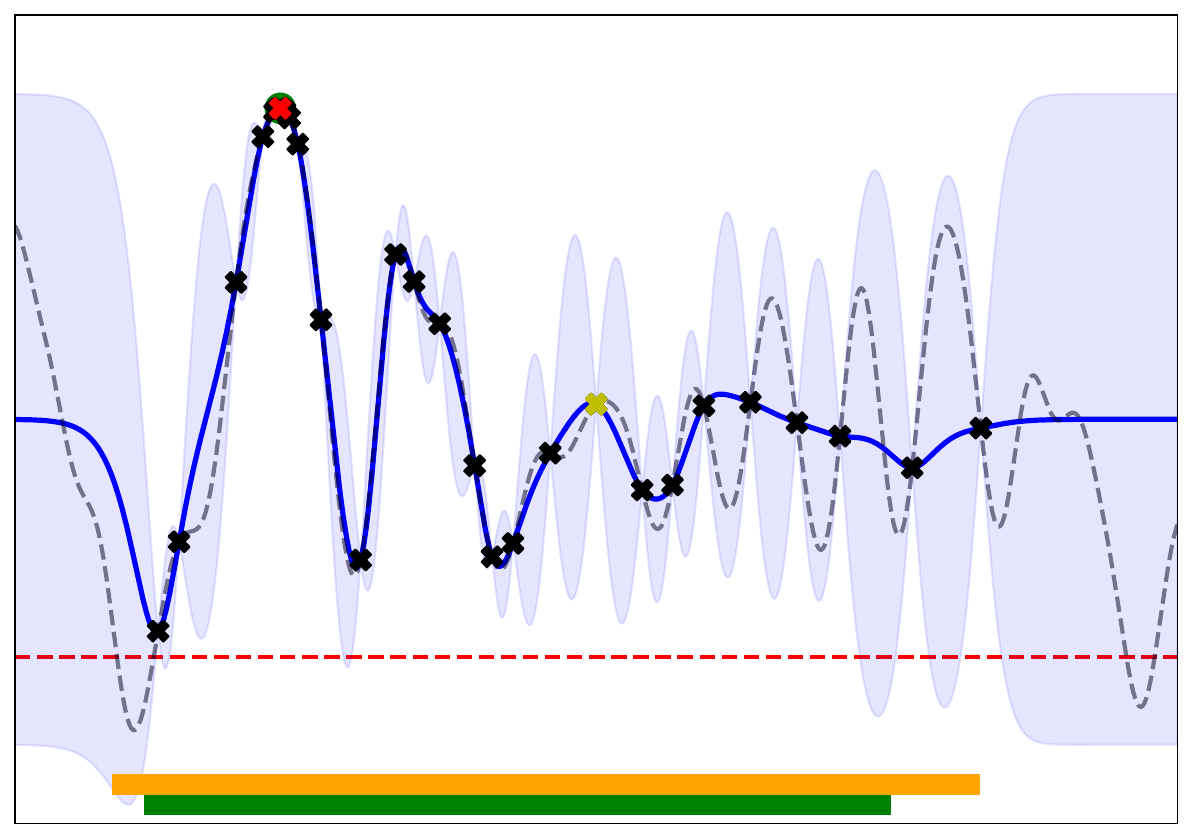}
        \caption{\autotuner at $\n=28$. Final iteration.}
        \label{fig:sopt_2}
    \end{subfigure}
    \caption{\autotuner illustration. (i) The grey, dashed line represents the true function. (ii) The red, dashed line represents the constraint. (iii) The blue line represents the Gaussian Process mean, and the shaded blue area represents the confidence bounds ($\mu_{\n}(\boldsymbol{\theta}) \pm \sqrt{\beta_{\n}} \sigma_{\n}(\boldsymbol{\theta})$).  (iv) The cross markers represent the samples, with yellow denoting the first sample and red the \autotuner recommended sample. (v) The green dot denotes the goal at each iteration.
    The algorithm learns the pessimistic (green bar) and optimistic (orange bar) sets, and explores the parameter space while satisfying the performance constraint. At $\n=5$, the goal is outside of the pessimistic set, although it is inside the optimistic set. \autotuner expands the pessimistic set by approaching the goal. It reaches the goal at $\n=15$, and by further expanding the sets, it discovers the maximum of the function at $\n=28$.}
    \vspace{-1.5em}
    \label{fig:coat_mpc_all}
\end{figure*}
\looseness -1 A visual representation of the pessimistic and optimistic operators evaluated at a single point is depicted in \cref{fig:lipschitz}. For notational convenience, we denote \mbox{$\reachOper[]{}(\safeset) \coloneqq \reachOper[0]{}(\safeset)$} when referring to $\epsilon=0$ case (analogously for the pessimistic and the optimistic operator as well). By applying these one-step constraint operators recursively, we next define the pessimistic, optimistic and reachability expansion operators, which are used to obtain the pessimistic and optimistic estimates of the true constraint set:
\begin{align}
&\tilReachOper[\epsilon]{\safeset} = \lim_{\m\rightarrow \infty} \ReachOperReach[\epsilon, \m]{\safeset},\label{eq:reach_oper}\\
&\tilPessiOper[]{\safeset} = \lim_{\m\rightarrow \infty} \PessiOperReach[\m]{\safeset},\label{eq:pessi_oper}\\  
&\tilOptiOper[\epsilon]{\safeset} = \lim_{\m\rightarrow \infty} \OptiOperReach[\epsilon, \m]{\safeset},\label{eq:opti_oper}  
\end{align} 
where \mbox{$\PessiOperReach[\m]{\safeset} := \pessiOper[]{\n}(\pessiOper[]{\n}\cdots(\pessiOper[]{\n}(\safeset)))$} and \mbox{$\OptiOperReach[\epsilon,\m]{\safeset} := \optiOper[\epsilon]{\n}(\optiOper[\epsilon]{\n}\cdots(\optiOper[\epsilon]{\n}(\safeset)))$} are the $\m$-step pessimistic and optimistic expansion operators. Using the expansion operators on $\pessiSet[]{\n-1}$ we obtain pessimistic \mbox{$\pessiSet[]{\n} = \tilPessiOper[]{\pessiSet[]{\n-1}}$} and optimistic \mbox{$\optiSet[]{\n} = \tilOptiOper[\epsilon]{\pessiSet[]{\n-1}}$} estimates of the true constraint set. 
Analogously, \mbox{$\ReachOperReach[\epsilon, \m]{\safeset} := \reachOper[\epsilon]{}(\reachOper[\epsilon]{}\cdots(\reachOper[\epsilon]{}(\safeset)))$} denotes the $\m$-step $\epsilon-$close true reachability operator. Applying the reachability operator from~\cref{eq:reach_oper} on the initial seed $\safeset_0$, we obtain the $\epsilon-$close true reachability set \mbox{$\constSet[,\epsconst]{} = \tilReachOper[\epsilon]{\safeset_0}$}, which includes all the parameters being at least $\epsconst-$conservative from violating the constraint.%

%% file: 4-algorithm.tex
% !TeX spellcheck = en_US
% !TeX encoding = UTF-8
% !TeX root = ../main.tex

\section{\autotuner}
\label{chp:Algorithm}

In this section, we present \autotuner, for the optimization of MPC cost function parameters while respecting~\cref{eq:constraint}. The algorithm is introduced in Algorithm \ref{alg:safe_epsilon_ucb} with its optimality guarantees deferred to \cref{chp:Theoretical_results}. 

\vspace{.2cm}

\mypar{Intuition} 
Our algorithm's goal is (i) to ensure the satisfaction of~\cref{eq:constraint} while (ii) converging to the optimal cost function weights with few tuning iterations. For the former part (i), we sample weights from the pessimistic set $\pessiSet[]{\n}$, which guarantees satisfying~\cref{eq:constraint} with high probability. To ensure the later part (ii), we set a goal $\goal_{\n}$ in the optimistic set $\optiSet[]{\n}$, outside of which the weights do not satisfy~\cref{eq:constraint}. Additionally, to converge to the optimal weights with fewer tuning iterations, we employ a goal-directed approach and use an expansion method (see~\cref{alg:safe_expantion}) towards this goal. We showcase a visual representation of \autotuner in the 1D setting in~\cref{fig:coat_mpc_all}.

 \begin{algorithm}
    \caption{Constrained Expansion (CE)}\label{alg:safe_expantion}
    \begin{algorithmic}[1]
      \State \textbf{Input:} $\pessiSet[]{\n-1}, \goal_{\n}$
      \State \textbf{Recommend:} $\argmin_{\btheta \in \pessiSet[]{\n-1}}{||\goal_{\n} - \btheta||_2, \text{ s.t. } w_{\n}(\btheta) \geq \epsilon}$ \label{alg-se:line:1}
    \end{algorithmic}
  \end{algorithm}  
 
\mypar{Constrained Expansion} The objective is to learn about the satisfaction of~\cref{eq:constraint} of the current goal $\goal_{\n}$, given the pessimistic set $\pessiSet[]{\n}$. \autotuner's expansion strategy recommends the closest point to the goal, with respect to the Euclidean distance, inside the pessimistic set and that is not $\epsilon$-accurate, i.e., the width of the confidence bounds \mbox{$w_{\n}(\btheta) = u_{\n}(\btheta) - l_{n}(\btheta)$} evaluated at the point is greater than or equal to $\epsilon$. Thus, \autotuner recommends parameters that satisfy the constraint and that are as close as possible to the goal while maintaining exploration through the statistical confidence $\epsilon$. If the algorithm is certain enough about the performance $\constrain$ of a parameter, it will not further explore it.

  \begin{algorithm}\caption{\autotuner}\label{alg:safe_epsilon_ucb}
    \begin{algorithmic}[1]
      \State \textbf{Input:} Initial seed $\seed$, $\constrain \sim \mathcal{GP}(\mu_0(\btheta), k_0(\btheta, \btheta'))$, $\tau$, $\Domain$, Lipschitz constant $L$
      \State $\pessiSet[]{0} \leftarrow \seed$, $\optiSet[]{0} \leftarrow \Domain$ \label{alg:line:2}
      \For {$\n = 1, \hdots, N_{max}$,}
      \State  $\goal_{\n} \leftarrow \argmax_{\btheta \in \optiSet[]{\n-1}} u_{\n-1}(\btheta)$ \label{alg:line:4}
      \If{$\goal_{\n}\in \pessiSet[]{\n-1}$ and $w_{\n-1}(\goal_{\n}) < \epsilon$} \label{alg:line:5}
      \State Terminate \label{alg:line:6}
      \ElsIf{$\goal_{\n} \notin \pessiSet[]{\n-1}$}\label{alg:line:7}
      \State $\btheta_{\n} \gets \textsc{CE}(\pessiSet[]{\n-1}, \goal_{\n})$\label{alg:line:10}
      \EndIf
      \State $y_{\n} \gets \constrain(\btheta_{\n}) + \eta_{\n}$ and Update GP \label{alg:line:8}
      \State $\pessiSet[]{\n} \gets \tilPessiOper[]{\pessiSet[]{\n-1}}$ \label{alg:line:13}
      \State $\optiSet[]{\n} \gets \tilOptiOper[\epsilon]{\pessiSet[]{\n-1}}$ \label{alg:line:14}
      \EndFor
     \State \textbf{Recommend:} $\goal_{\n}$ 
    \end{algorithmic}
  \end{algorithm}
  
\mypar{\autotuner} The pseudocode is summarized in \cref{alg:safe_epsilon_ucb}. We start the algorithm from parameters within $\safeset_0$, where the constraint \mbox{$\constrain(\btheta) \geq \tau, \forall \btheta \in \safeset_0$} is known to be satisfied due to~\Cref{assump:safe_seed}. In~\Cref{alg:line:2}, we initialize the pessimistic set to the initial seed $\safeset_0$ and the optimistic set to the parameters' domain $\Domain$. We next compute the goal $\goal_{\n}$ within the optimistic set using the Upper Confidence Bounds (UCB) (\Cref{alg:line:4}). The objective is to reach the goal while ensuring that sufficient information is gained for exploration. This is measured using the width of the confidence bounds $w_{\n-1}(\btheta)$. 
Based on the location of the goal and the uncertainty about it,  \autotuner considers the following two cases:

\begin{enumerate}
    \item \textbf{The goal is in the pessimistic set with the uncertainty below} $\boldsymbol{\epsilon}$, i.e., $w_{n-1}(\goal_{\n}) < \epsilon$: This case represents that the goal satisfies~\cref{eq:constraint} and the performance function value is known up to the desired confidence. Since the goal defined using UCB criteria is an upper bound over the possible locations that fulfill the constraint, we terminate the algorithm (\cref{alg:line:6}) with desired guarantees (\cref{chp:Theoretical_results}).
    
    \item \textbf{The goal is not in the pessimistic set}: In order to expand the pessimistic set to the goal location we use \emph{constrained expansion} defined in~\cref{alg:safe_expantion} (~\Cref{alg:line:10}) and obtain a sampling location. Then we collect a noisy measurement of $\constrain$ and update the posterior (~\Cref{alg:line:8}).
\end{enumerate}
If none of the two cases holds,  the goal is within the pessimistic set but the uncertainty \mbox{$w_{n-1}(\goal_{\n}) \geq \epsilon$}. In this case,  we directly collect a noisy measurement of $\constrain$ to reduce the uncertainty, update the posterior and continue the process to get the new goal location (\Cref{alg:line:8}). Finally, if the algorithm has not terminated, the pessimistic and optimistic sets are updated using their expansion operators (Lines \ref{alg:line:13} and \ref{alg:line:14}). This process expands the potential set of parameters that fulfill~\cref{eq:constraint} and explores the parameter domain. Note that low values of $\epsilon$ lead to a slow convergence and increased exploration, while high values of $\epsilon$ have to opposite effect.

%% file: 5-theoretical_results.tex
% !TeX spellcheck = en_US
% !TeX encoding = UTF-8
% !TeX root = ../main.tex
\section{Theoretical analysis}
\label{chp:Theoretical_results}

In this section, we present our core theoretical result, i.e., convergence to optimal parameters while satisfying~\cref{eq:constraint} in finite time with arbitrarily high probability. 
We start with the following assumption for the finite sample complexity.
\begin{assumption} \label{assump:sublinear}
    $\betaconst[\n]\!\gamma_{\n}$ grows sublinear in $\n$, i.e., \mbox{$\!\betaconst[\n]\gamma_{\n} \!<\! \!\mathcal{O}(n)$}. 
\end{assumption}

This assumption is common in most prior works \cite{prajapat2022near,prajapat2024safe,goose} aimed to establish sample complexity or sublinear regret results and are not restrictive. 
It can be satisfied for commonly used kernels, e.g., linear kernels, squared exponential, Mat\'ern, etc., with sufficient eigen decay \cite{vakili2021information,Srinivas2009GaussianPO} under the bounded $B_q$ of~\Cref{assump:q_RKHS}.

\begin{restatable}{theorm}{restateMainThm}
\label{thm:convergence} 
Let~\Cref{assump:safe_seed,assump:q_RKHS,assump:sublinear} hold and $\n^{\star}$ be the largest integer such that \mbox{$\frac{\n^{\star}}{\betat[\n^{\star}] \gamma_{\n^{\star}}} \leq \frac{C_1}{\epsilon^2}$} with \mbox{$C_1 \coloneqq 8/\log(1+\sigma^{-2}_\eta)$}. With probability at least $1-\delta$, \autotuner satisfies~\cref{eq:constraint} for the evaluated parameters \mbox{$\btheta_{\n}, \forall \n\geq 1$} and the closed-loop system of~\cref{eq:MPCFormulation} satisfies state and input constraints for all times $t \geq 0$. Moreover, \mbox{$\exists \n \leq \n^{\star}$} satisfying,
\begin{align*}
    \constrain(\btheta_{\n}) \geq \max_{\btheta\in \constSet[,\epsconst]{}} \constrain(\btheta) - \epsilon.
\end{align*}
with probability at least $1-\delta$.
\end{restatable}

The proof is in \cref{chp:proof}. Thus, \autotuner guarantees satisfying~\cref{eq:constraint} with high probability at each tuning iteration. Moreover, it ensures finite time convergence to the reachable optimal tuning parameters under performance constraints. Intuitively, ~\cref{eq:constraint} is ensured by executing the parameters from the pessimistic set; while the other state and input constraints are ensured directly by MPC (see \cref{eq:MPCFormulation}). Since the final recommended tuning parameter is as per \ucb (optimistic estimate of performance) with uncertainty below $\epsilon$, we are guaranteed to converge to the optimal solution reachable under performance constraints. In contrast to the earlier sample complexity results \cite{safeopt_15} in discrete domains, we do not have an explicit dependence on the domain size $|\tilReachOper[\epsilon]{\safeset}|$. This makes our bound tighter and more usable for larger domains (or finer discretization). We achieved this by extending the sample complexity analysis of \cite{prajapat2024safe} from continuous domains to discrete domains. In particular, we show that the pessimistic set, $\pessiSet[]{\n}$ formed using discrete expansion operators (\cref{eq:pessi_oper}) is a subset of the pessimistic set of the continuous domain \cite{prajapat2024safe}, and additionally with our sampling rule (\cref{alg-se:line:1}) the same tighter bound of \cite{prajapat2024safe} holds in the discrete domain as well.

%% file: 6-experimental_results.tex
% !TeX spellcheck = en_US
% !TeX encoding = UTF-8
% !TeX root = ../main.tex
\vspace{-.4em}
\section{Experimental results}
\label{chp:Experimental_results}

\begin{table*}[t]
    \centering
        \vspace{.2cm}
        \begin{tabular}{l|cccc|cccc}
            \toprule
            & \multicolumn{4}{c|}{\textbf{Simulation}} & \multicolumn{4}{c}{\textbf{RC platform}} \\ \cline{2-9}
            \makecell{\textbf{Algorithm}} &\makecell{\#Constraint \\ violations} &\makecell{Min. \\lap time[s]} &\makecell{Mean lap time \\ $\pm$ std. deviation[s]} &\makecell{Number of \\ iterations}&\makecell{\#Constraint \\ violations} &\makecell{Min. \\lap time[s]} &\makecell{Mean lap time \\ $\pm$ std. deviation[s]} &\makecell{Number of \\ iterations}
            \\ \midrule
            \makecell{\autotuner}         &\makecell{$\mathbf{0}$}       &\makecell{$\mathbf{4.68}$}        &\makecell{$5.57 \pm 0.57$} &\makecell{$\mathbf{28}$}        &\makecell{$\mathbf{0.33}$}       &\makecell{$\mathbf{6.49}$}        &\makecell{$7.55 \pm 0.50$} &\makecell{$\mathbf{21.66}$} \\ 
            \makecell{\safeopt}         &\makecell{$1.8$}       &\makecell{$4.71$}        &\makecell{$5.84 \pm 0.82$} &\makecell{70} &\makecell{$5.33$}       &\makecell{$6.52$}        &\makecell{$7.88 \pm 0.74$} &\makecell{70} \\
            \makecell{\gpucb}         &\makecell{$5.2$}       &\makecell{$\mathbf{4.68}$}        &\makecell{$\mathbf{5.03 \pm 1.02}$} &\makecell{70} &\makecell{$7.0$}       &\makecell{$6.82$}        &\makecell{$\mathbf{7.51 \pm 0.77}$} &\makecell{70} \\
            \makecell{\wml}         &\makecell{$2.83$}       &\makecell{$4.71$}        &\makecell{$5.30 \pm 0.80$} &\makecell{70}&\makecell{$6.66$}       &\makecell{$6.95$}        &\makecell{$7.60 \pm 0.68$} &\makecell{70} \\
            \makecell{\eic}         &\makecell{$7.0$}       &\makecell{$\mathbf{4.68}$}        &\makecell{$5.24 \pm 1.16$} &\makecell{70} & \makecell{$\_$} & \makecell{$\_$} & \makecell{$\_$} & \makecell{$\_$}\\
            \makecell{\crbo}         &\makecell{$16.4$}       &\makecell{$4.71$}        &\makecell{$5.55 \pm 1.39$} & \makecell{70} & \makecell{$\_$} & \makecell{$\_$} & \makecell{$\_$} & \makecell{$\_$}\\
            \bottomrule
        \end{tabular}
        \vspace{-.1cm}
    \caption{Averaged results of \autotuner and baselines. The algorithms were evaluated 5 times in simulation and 3 times with the RC platform.}
        \vspace{-.1cm}
    
    \vspace{-2em}
    \label{tab:comparisons}
\end{table*}

We present an extensive evaluation of \autotuner in an autonomous racing application. Our evaluation is conducted using an autonomous racing simulation and the RC platform of~\cref{fig:crs_car}. We first discuss the MPC formulation used in our experiments in~\cref{subsec:mpc}, followed by our experimental setup in~\cref{subsec:exp_setup}. Finally, we present the simulation and experimental results in~\cref{subsec:sim,subsec:car}.

\begin{figure}[h]
        \centering
        \includegraphics[width=0.6\columnwidth]{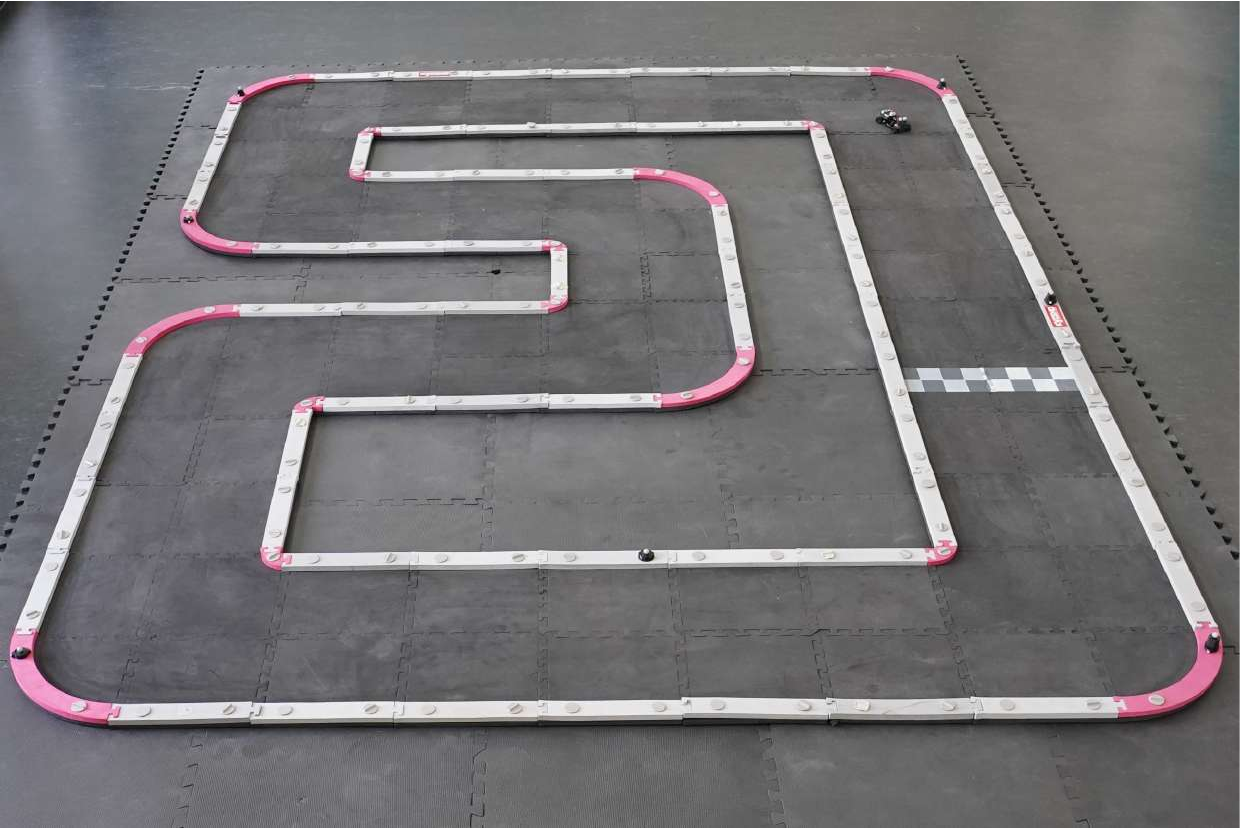}
        \caption{1:28 scale RC racecar~\cite{Carron2022ChronosAC} and track used in the experiments.}
        \label{fig:crs_car}
        \vspace{-1.5em}
\end{figure}

\subsection{Model predictive contouring control (MPCC)}
\label{subsec:mpc}
\looseness -1 In our experiments, we use a MPCC ~\cite{Liniger2015OptimizationbasedAR} formulation, which has been successfully applied for autonomous racing in a known track.  In particular, our MPCC is formulated as:
\begin{equation}    
\begin{aligned}
 \min_{\bu_{\mathrm{0:N}}} \quad &  \sum_{i=0}^{N}{-q_{\lambda} \lambda_i + \boldsymbol{e}_i^T \bbigq \boldsymbol{e}_i + \bu_i^T \bbigr\bu_i}\\\
    \textrm{s.t.} \quad & \bx_0 = \bxhat(t), \, \bx_{i+1} = \dyn_d(\bx_i, \bu_i) \\
    &\bx_i \in \xset, \, \bu_i \in \uset, \, \boldsymbol{g}_i(\bx_i, \bu_i) \in \mathcal{G},
    \label{eq:MPCCFormulation}
\end{aligned}
\end{equation}

\noindent \looseness -1 where \mbox{$\bx_i=[x_i, y_i, \psi_i, v_{x_i}, v_{y_i}, \dot{\psi}_i]$} is the state of the car (including position, orientation, and velocities), \mbox{$\bu_i=[\delta_i, T_i]$} is the control input (steering angle and drivetrain command), $\lambda$ is the parameter that determines the progress along the reference trajectory, $\boldsymbol{e}_i$ denotes contour and lag errors,~$\dyn_d(\cdot, \cdot)$ is the nominal car model consisting of a Pacejka dynamic bicycle model~\cite{PACEJKA2006v}, and $\boldsymbol{g}_i(\bx_i, \bu_i)$ are linear and nonlinear constraints on the states and inputs. The matrix \mbox{$\bbigq=\text{diag}([Q_{contour}, Q_{lag}])$} controls the longitudinal and lateral deviation from the reference trajectory, $q_{\lambda}$ regulates the progress of the car, and $\bbigr$ determines the smoothness of the inputs. 

This MPC aims to maximize progress while penalizing deviations from the reference trajectory (we use the middle trajectory as a reference).  Note that this MPC formulation does not directly minimize time. However, we set lap time as the performance function of \autotuner to achieve lap time minimization.  We set $N=40$, $q_{\lambda}=3.3$, 
  $\bbigr=\text{diag}([0.3, 0.1])$ and make use of the boundary constraints as introduced in~\cite{Liniger2015OptimizationbasedAR} with a track width of $0.46m$.

\subsection{Experimental setup}
\label{subsec:exp_setup}
To quantify the performance of the MPC, we define the performance function $\constrain$ as the negative lap time of a single flying lap, where the car does not start from a stationary position. The negative sign is introduced to reflect the objective of minimizing lap time. Additionally, we establish the performance upper bound as \mbox{$\tau = \tau_{scale}(\constrain(\safeset_0[0]) +\eta)$}, where \mbox{$\tau_{scale} \geq 1$} is a user-specified scaling factor, and \mbox{$\constrain(\safeset_0[0]) +\eta$} represents a noisy evaluation of the negative lap time, due to different errors while running in the real-world, e.g., state estimation or process noises. This noisy laptime is obtained using the initial seed parameters, which are known a priori from manual tuning. Hence, we constrain the lap time to always be lower than the initial lap time multiplied by a scaling factor that is larger than one.

We conduct a comparative analysis of our proposed method with several unconstrained methods, namely,  \gpucb~\cite{Srinivas2009GaussianPO}, Weighted Maximum Likelihood (\wml)~\cite{Romero2022WeightedML}, and Confidence Region Bayesian Optimization (\crbo)~\cite{pmlr-v144-frohlich21a}.  Additionally, we evaluate our method against constrained optimization methods, specifically, (\eic)~\cite{Gardner2014BayesianOW} and \safeopt~\cite{safeopt_15}.

In our experiments, we jointly optimize $Q_{contour}$ and $Q_{lag}$ (see~\cref{eq:MPCCFormulation}). We uniformly discretize the parameters' domain into 10,000 combinations within the range of $\left[ 0,\, 1000 \right]^2$. These combinations are normalized to $\left[ 0,\, 1 \right]^2$. Furthermore, we set the initial weights of $Q_{contour}$ and $Q_{lag}$ to 500. For the methods that use a Gaussian Process to model the performance function, we select a Matérn Kernel with a smoothness parameter of \mbox{$\nu = 5/2$}. A unique length-scale of \mbox{$l=0.1$} is chosen for both dimensions and we use \mbox{$\beta = 5.0$}.

\subsection{Simulation results}
\label{subsec:sim}

We present a comprehensive evaluation of our method in comparison to the baselines over a total of 70 iterations, during which the methods are permitted to sample and assess 70 distinct parameters. Note that, in this setup, \autotuner takes less than 70 iterations due to its termination criteria. As presented in~\cref{tab:comparisons} and \cref{fig:cum_regret}, our method outperforms baselines in terms of performance constraint violations while converging to the optimal parameters in 30 iterations. We observe that \autotuner achieves the same lap time as \gpucb, indicating that both algorithms converge to the optimal parameters. However, \gpucb, as well as other baselines, show slightly better cumulative regret over the initial 30 iterations (see~\cref{fig:simu}). This is mainly due to the simulation domain $D$, where a large part satisfies~\cref{eq:constraint} (see~\cref{fig:samples_results}). As a result, unconstrained baselines can converge faster to the optimal parameters
since they do not take into account the constraint. However, they
still violate~\cref{eq:constraint}. Although cumulative regret indicates convergence, our focus is on achieving the best lap time as quickly as possible without violating~\cref{eq:constraint} and in all these respects, \autotuner outperforms the baselines. Note that we tested \autotuner in multiple tracks and achieved a similar performance. Due to space constraints, we only present the results of one track.

\begin{figure}[h]
  \vspace{-.5em}
    \centering
    \begin{subfigure}{0.49\columnwidth}
    \includegraphics[width=\columnwidth]{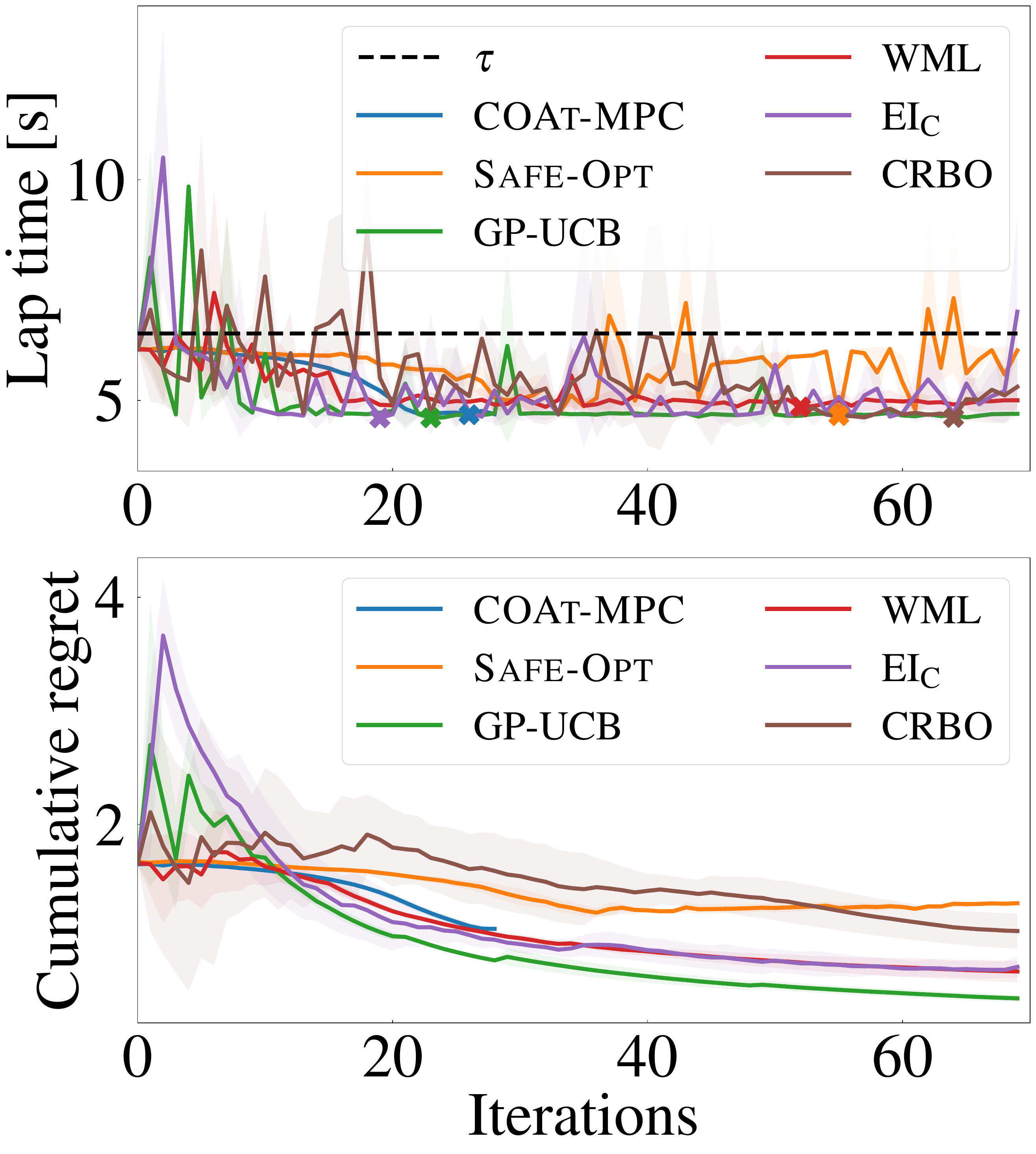}
    \caption{Simulation results.} \label{fig:simu}
    \end{subfigure}
    \begin{subfigure}{0.49\columnwidth}
    \includegraphics[width=\columnwidth]{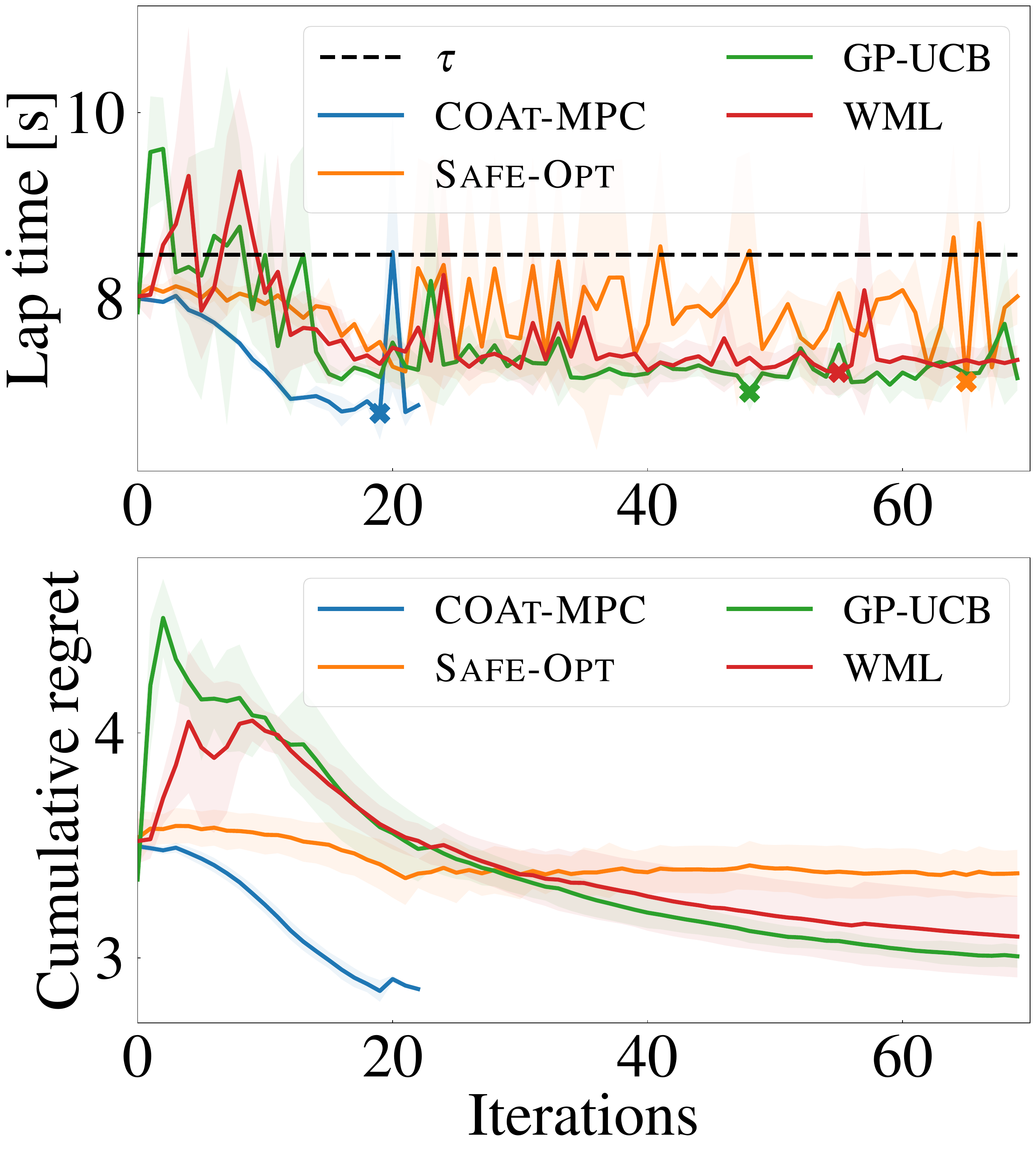}
    \caption{RC platform results.}
    \end{subfigure}
  \caption{Lap time and cumulative regret over time with their standard deviation. \autotuner converges in only 20-30 iterations and achieves the lowest cumulative regret with the RC platform. Cross markers indicate each method's minimum laptime. The initial laptime (iteration 0) is the same for all methods.}
  \vspace{-0.5em}
  \label{fig:cum_regret}
\end{figure}

\cref{fig:samples_results} illustrates the tuning process of our method and baselines. \autotuner starts by sampling close to the initial seed, aiming to expand the pessimistic set. Next, our method cautiously samples parameters that result in an improved lap time, and converges once it is $\epsilon$-certain that it has found the optimal parameters. \autotuner performs a goal-directed exploration to converge to the optimal parameters. Our method requires less exploration to converge compared to the baselines and always satisfies the performance constraint.

\subsection{RC platform results}
\label{subsec:car}

After observing the outcomes in simulation, our algorithm is benchmarked against \safeopt, \gpucb, and WML, which proved to be the best among all baselines in terms of constraint violations. As shown in Table \ref{tab:comparisons}, our approach surpasses the baselines in terms of constraint violations and closely approximates the average mean lap time of \gpucb, while effectively converging to the optimal parameters. Note that \autotuner violates~\cref{eq:constraint} once in our experiments. During our experiments, external factors such as odometry noise influenced performance, potentially causing the car to crash and violate the constraints. The \autotuner's constraint violation could be attributed to one of these external factors. 

Our method stands out by achieving the lowest cumulative regret and converging to the optimal parameters in just 20-25 iterations, as shown in~\cref{fig:cum_regret} and~\cref{tab:comparisons}. Additionally, as demonstrated in the regret plot of~\cref{fig:cum_regret}, \autotuner achieves a lower minimum regret as compared to the baselines, which indicates that our method is capable of converging to the optimal set of parameters while satisfying~\cref{eq:constraint} with high probability. Furthermore, \autotuner reaches the lowest cumulative regret as a consequence of its sampling taking into account~\cref{eq:constraint}, which is important due to a small true safe set. Other baselines struggle with the small safe set and violate the constraint, leading to a worse cumulative regret. \cref{fig:samples_results} illustrates the tuning process of all methods.

\begin{figure}[h]
    \vspace{.1cm}
    \begin{tabular}{cc}
        \multicolumn{2}{c}{\textbf{Simulation}}\\
        \includegraphics[width=.48\columnwidth]{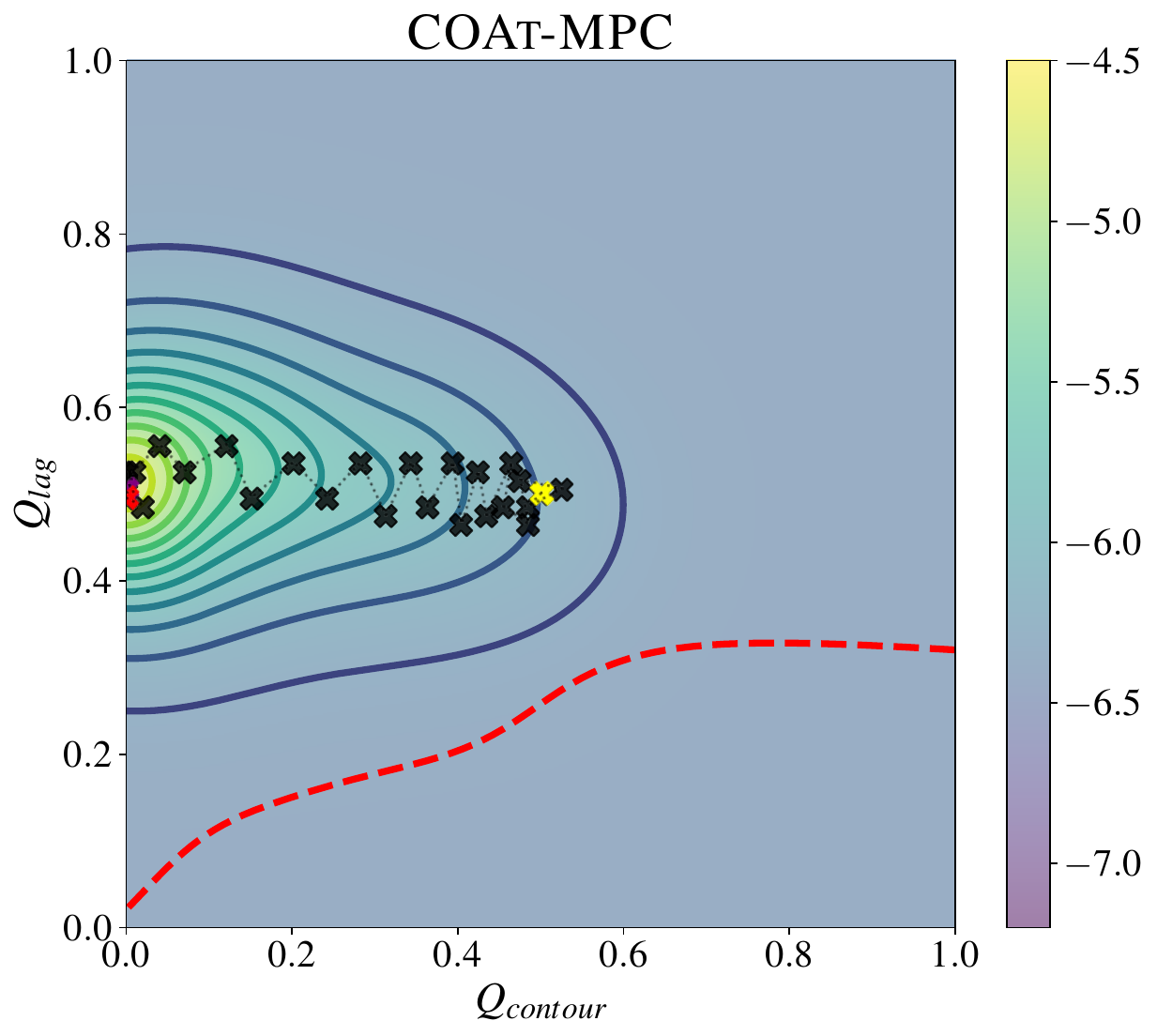} &
        \includegraphics[width=.48\columnwidth]{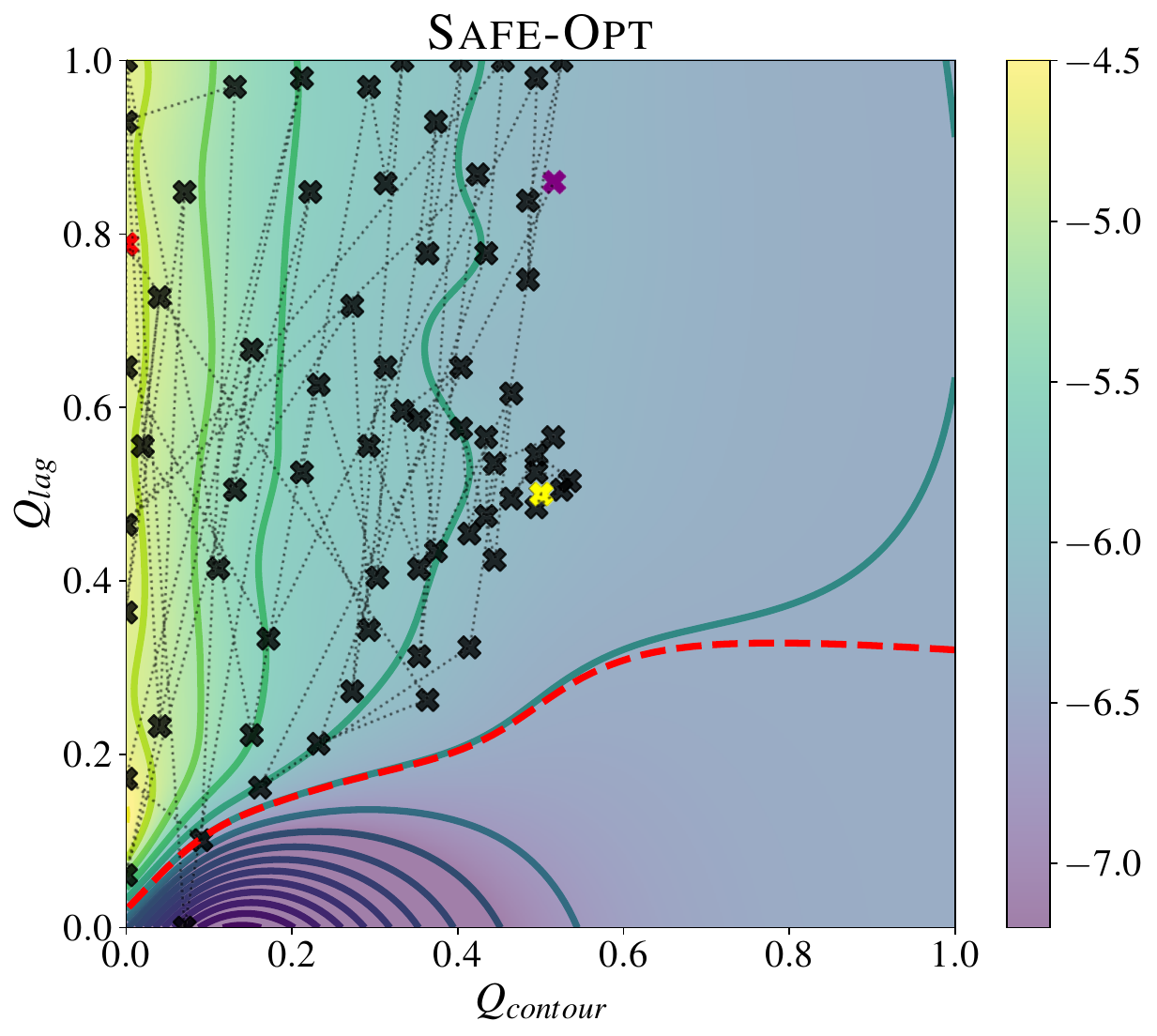} \\ 
        \includegraphics[width=.48\columnwidth]{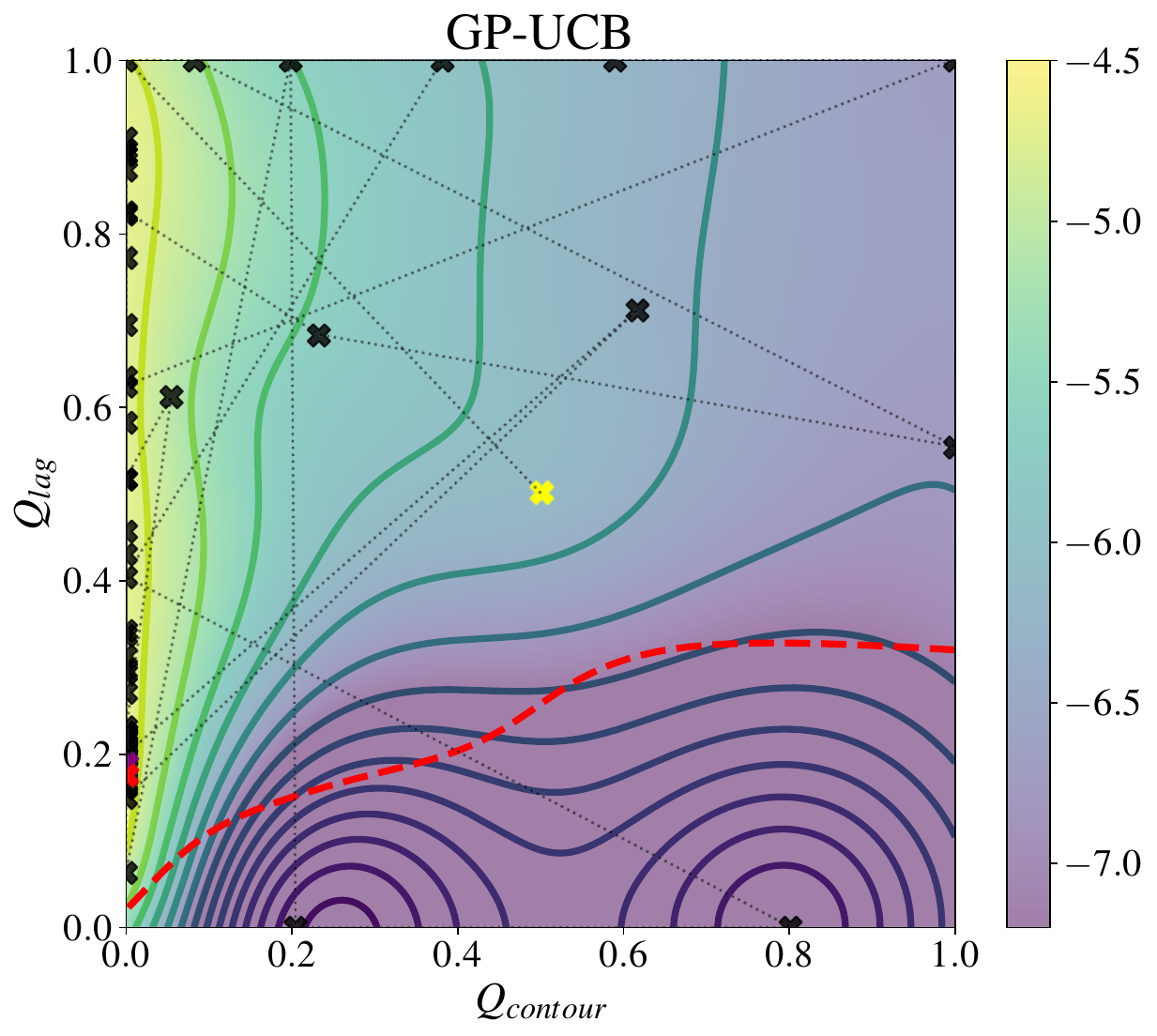}&
        \includegraphics[width=.48\columnwidth]{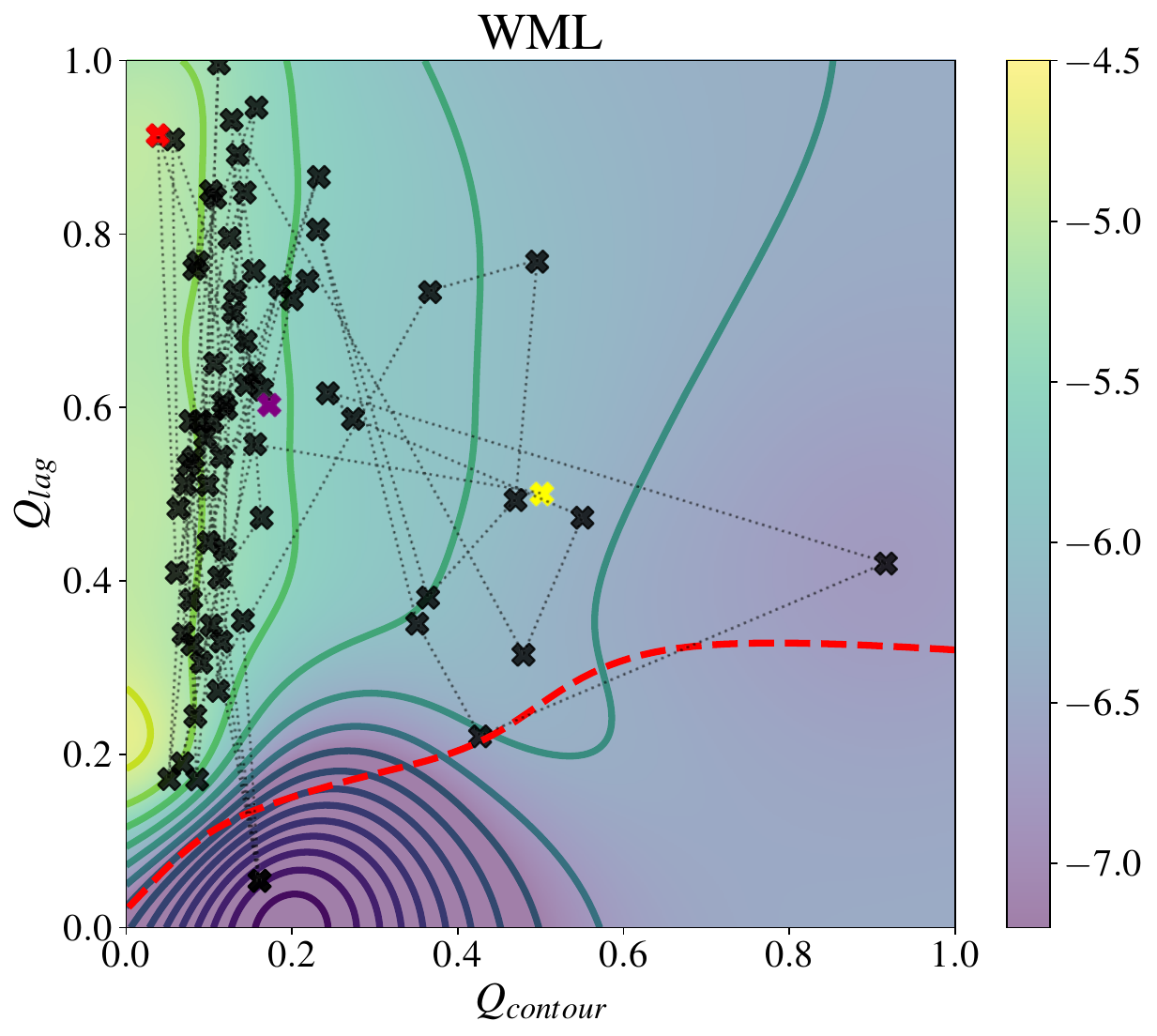} \\
        \multicolumn{2}{c}{\textbf{RC platform}}\\
        \includegraphics[width=.48\columnwidth]{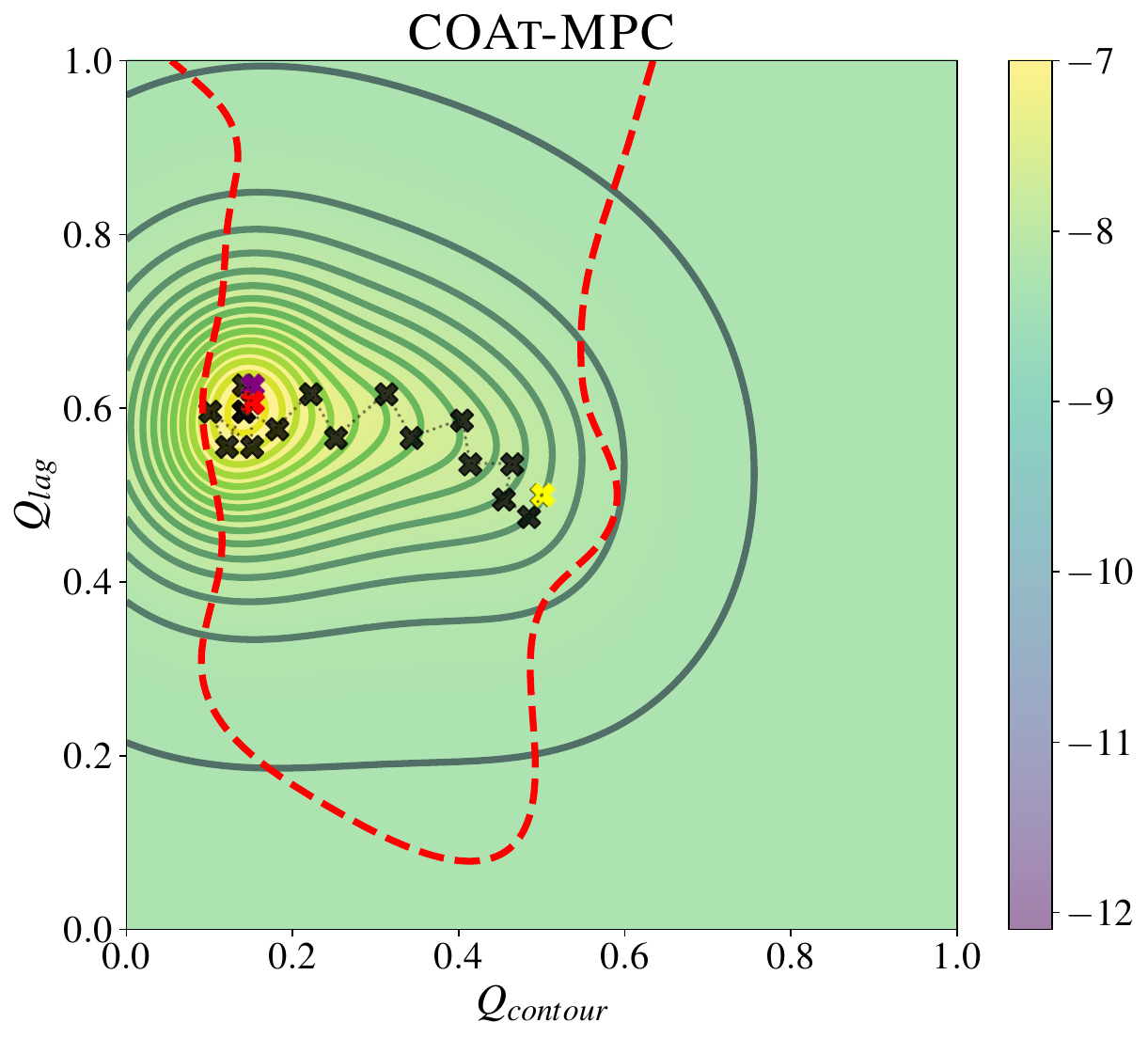} &
        \includegraphics[width=.48\columnwidth]{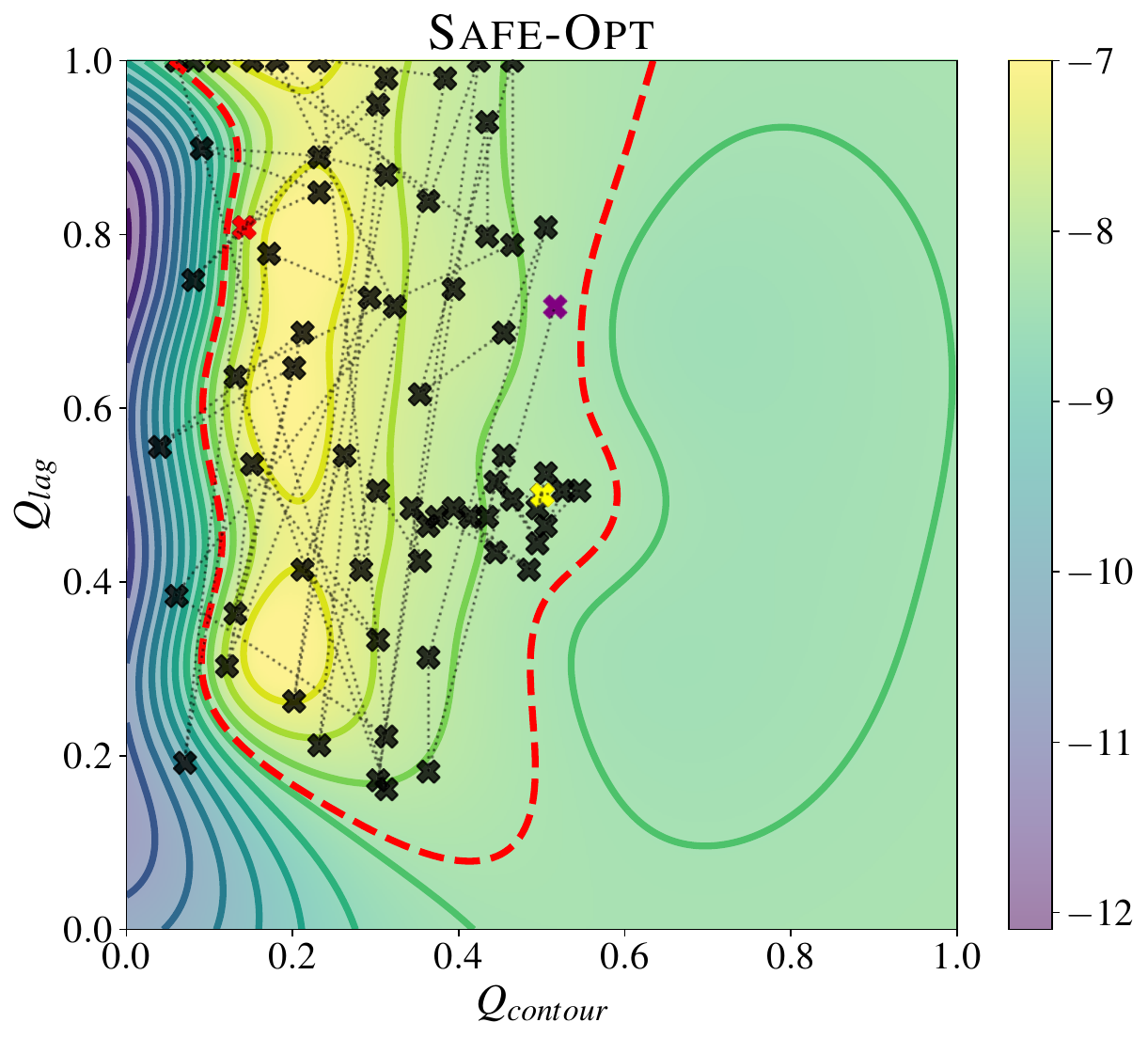} \\
        \includegraphics[width=.48\columnwidth]{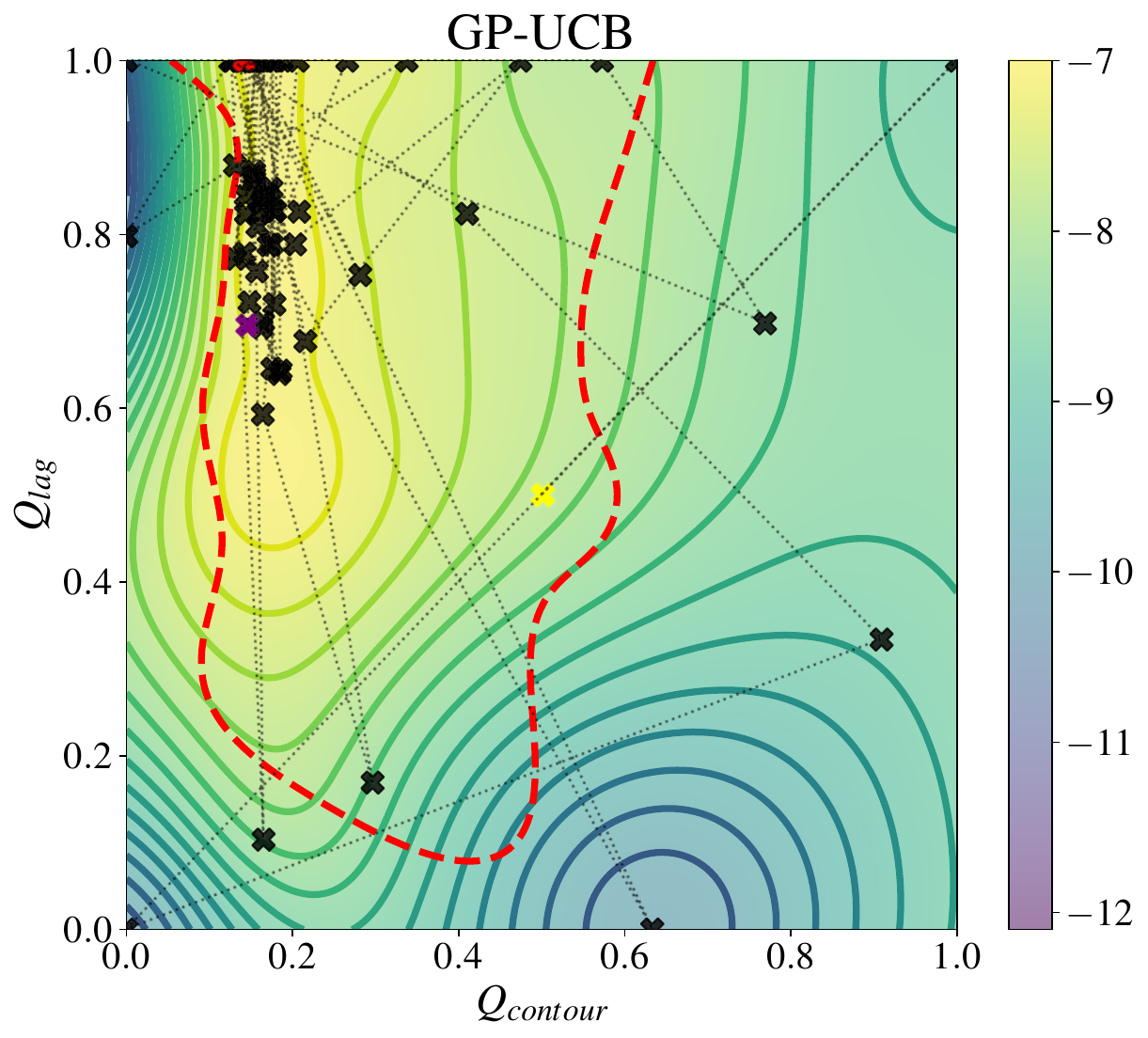} &
        \includegraphics[width=.48\columnwidth]{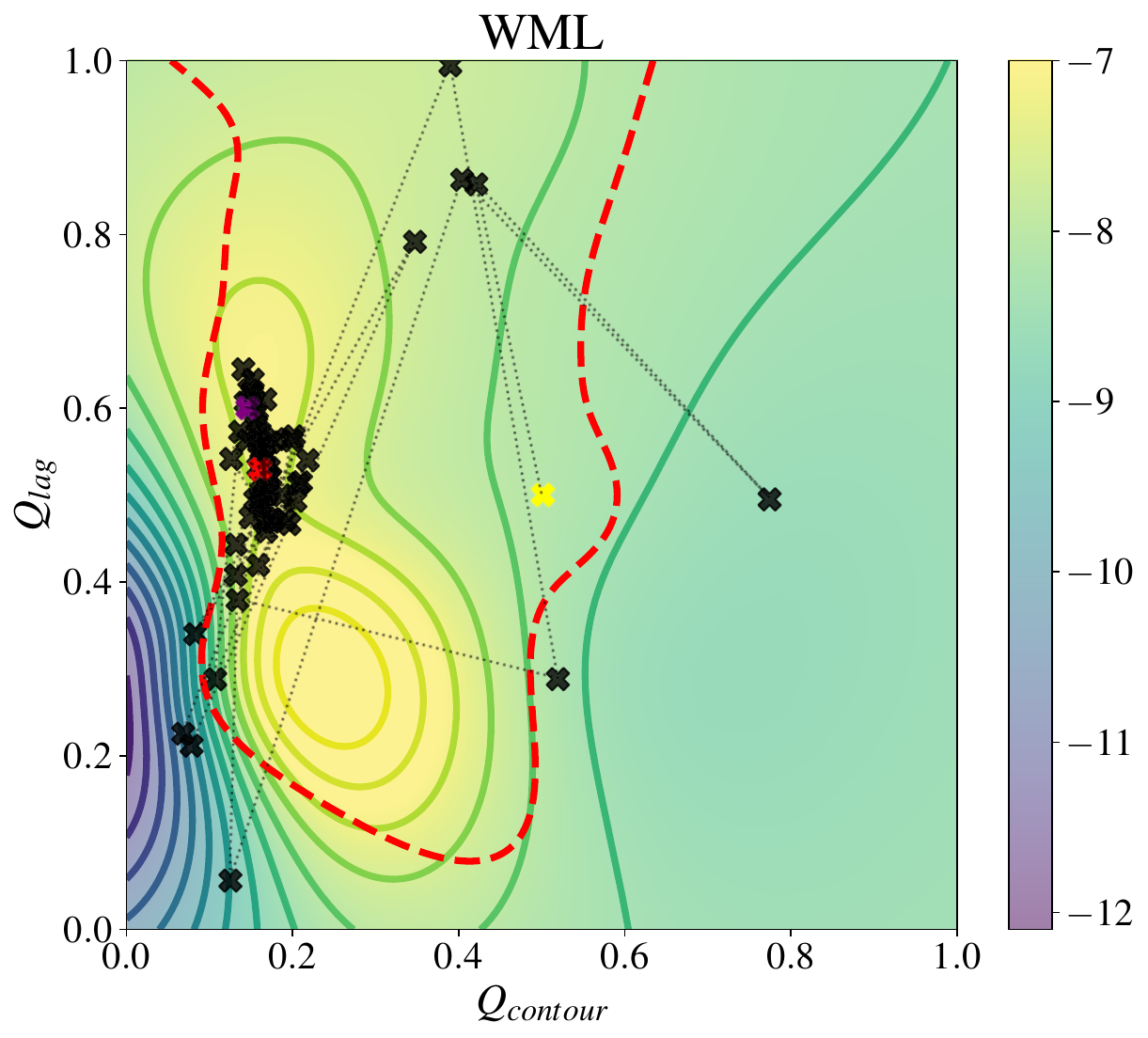}
        \end{tabular}
        \vspace{-.1cm}
    
    \caption{Methods' tuning process. The black marker denotes the method's samples. The black lines denote the sample trajectories. The yellow marker denotes the initial values of $Q_{contour}$ and $Q_{lag}$, the purple marker is the last sample, and the red marker is the best lap time sample. The red line denotes the area where $\constrain(\btheta) = \tau$. The heatmap is the GP posterior mean.}
    \vspace{-0.5em}
    \label{fig:samples_results}
    \end{figure}

%% file: 7-conclusions.tex
% !TeX spellcheck = en_US
% !TeX encoding = UTF-8
% !TeX root = ../main.tex
\section{Conclusions}
\label{chp:Conclusions}

We propose \autotuner, a method for MPC tuning that guarantees a performance constraint satisfaction with arbitrarily high probability. Our approach leverages the assumption of Lipschitz continuity in the objective function to construct pessimistic and optimistic constraint sets. We use the optimistic set to define a goal location at each iteration, while we restrict our recommendations to be within the pessimistic set. We present a theoretical analysis of our method, conclusively demonstrating its ability to achieve optimal tuning parameters in finite time while guaranteeing the satisfaction of the performance constraint with arbitrarily high probability. Additionally, our evaluation against state-of-the-art methods demonstrates that our method outperforms them in terms of the number of constraint violations, as well as in cumulative regret over time, in the context of MPC tuning for autonomous racing. Finally, an interesting line of future work could be to extend the method to work in high-dimensional parameter space. Our approach discretizes the parameter space to compute the pessimistic and optimistic sets, which leads to exponential space complexity.

%% file: A2-sample-complexity.tex
\subsection{Auxiliary lemmas for Proof of Theorem \ref{thm:convergence}}
\label{chp:proof}
In this section, we develop tools for the theoretical analysis of \autotuner, present auxiliary lemmas, and finally use them to prove Theorem~\ref{thm:convergence}.

To simplify analysis, we define a pessimistic set, \mbox{$\pessiSet[,\sempc]{\n} \coloneqq \{\btheta \in \Domain| \exists \btheta' \in \Domain, \lbconst[\n](\btheta') - L d(\btheta, \btheta')  \geq \tau \}$},
motivated from \cite{prajapat2024safe}.  
Note that \mbox{$\pessiSet[,\sempc]{\n}$} does not explicitly depend on the initial safe set and is defined over the domain, precisely,  \mbox{$\pessiSet[,\sempc]{\n} = \pessiOper[]{\n}(\Domain)$}.  
The following lemma establishes that our pessimistic set, $\pessiSet[]{\n}$ is subset of the one in \cite{prajapat2024safe}.
\begin{lemma}
    $\pessiSet[]{\n} \subseteq \pessiSet[,\sempc]{\n}, \forall \n \geq 0$. \label{lem:pessi_subset}
\end{lemma}\vspace{-1.1em}
\begin{proof} 
Note that $\pessiSet[]{\n} = \tilPessiOper[]{\pessiSet[]{\n-1}}$ and using \cref{eq:pessi_oper}, we get, $\pessiSet[]{\n} =\lim_{\m \rightarrow \infty} \PessiOperReach[\m]{\pessiSet[]{\n-1}}$. This implies $\pessiSet[]{\n} = \pessiOper[]{\n} (\lim_{\m \rightarrow \infty} \PessiOperReach[\m]{\pessiSet[]{\n-1}}) \subseteq \pessiOper[]{\n} (\Domain) = \pessiSet[,\sempc]{\n}$. The last equality follows by definition of the set, $\pessiSet[,\sempc]{\n}$.   
\end{proof}
\begin{corollary} [Theorem 1 \cite{prajapat2024safe}] 
\label{lem:sagempc}
\looseness -1 Let \cref{assump:q_RKHS,assump:sublinear} hold and $\n^{\star}$ be the largest integer satisfying $\frac{\n^{\star}}{\beta_{\n^{\star}} \gamma_{\n^{\star}}} \leq \frac{C}{\epsconst^2}$, with $C = 8/ \log (1 + \noiseconst)$. The sampling scheme $\btheta_\n \in \pessiSet[]{\n-1} : \sumMaxwidth[]{\n-1}(\btheta_\n) \geq \epsconst$ satisfy $\constrain(\btheta_\n)\geq \tau, \forall \n \geq 1$ with probability at least $1-\delta$ and $\exists \n \leq \n^\star : \forall \btheta \in \pessiSet[]{\n},~\sumMaxwidth[]{\n}(\btheta)<\epsconst$. \label{thm:sample_complexity}    
\end{corollary}\vspace{-0.5 em}
\begin{proof}
     In \cite{prajapat2024safe}, SAfe Guaranteed Exploration using Model Predictive Control (SageMPC) uses a sampling rule $\btheta_\n \in \pessiSet[,\sempc]{\n-1} : \sumMaxwidth[]{\n-1}(\btheta_\n) \geq \epsconst$ which aligns with our sampling rule of  $\sumMaxwidth[]{\n-1}(\btheta_\n) \geq \epsconst$ (\cref{alg-se:line:1}). Moreover, Theorem 1 \cite{prajapat2024safe}, i.e, $\exists \n \leq \n^\star : \forall \btheta \in \pessiSet[,\sempc]{\n},~\sumMaxwidth[]{\n}(\btheta)<\epsconst$, and \cref{lem:pessi_subset}, i.e., $\pessiSet[]{\n-1} \subseteq \pessiSet[,\sempc]{\n-1}$ implies $\exists \n \leq \n^\star : \forall \btheta \in \pessiSet[]{\n},~\sumMaxwidth[]{\n}(\btheta)<\epsconst$, which ensures finite time convergence.
     
     Next we prove satisfaction of performance constraint \cref{eq:constraint}. Note that, $\forall \btheta \in \pessiSet[]{\n} \subseteq \pessiSet[,\sempc]{\n} \implies \exists \btheta' \in \Domain: \lbconst[\n](\btheta') - L d(\btheta, \btheta') \geq \tau \implies \constrain(\btheta)\geq \tau$ with probability at least $1-\delta$ using \cref{cor:beta}. 
\end{proof}\vspace{-0.5 em}
Thus, using Theorem 1 \cite{prajapat2024safe}, \autotuner ensures high probability satisfaction of \cref{eq:constraint} at every sampling location and 
 guarantees termination of the process within $n^\star$ iterations. 
 Next, we prove the optimality guarantees for \autotuner.

\vspace{-0.5 em}
\begin{lemma} \looseness -1 Let \cref{assump:q_RKHS} hold and \mbox{$\sumMaxwidth[]{\n-1}(\btheta^g_\n) < \epsilon$}, where, \mbox{$\btheta^g_\n \coloneqq \argmax_{\btheta \in \optiSet[]{\n-1}} \ubconst[\n-1](\btheta)$}.
Then with probability at least \mbox{$1-\delta$} it holds that, \mbox{$\constrain(\btheta^g_\n) \geq \max_{\btheta\in \constSet[,\epsconst]{}} q(\btheta) - \epsilon$}. \label{lem:near_optimality}
\end{lemma} \vspace{-0.75em}
\begin{proof} By construction of confidence bounds \eqref{eq:conf_bounds}, it follows that $\constrain(\btheta) \leq \ubconst[\n-1](\btheta), \forall \n \geq 1, \btheta \in \Domain$ with probability at least $1-\delta$. This implies $\constSet[,\epsconst]{} \coloneqq \tilReachOper[\epsilon]{\safeset_0} \subseteq \optiSet[]{\n-1}$. Next, given $\sumMaxwidth[]{\n-1}(\btheta^g_\n) < \epsconst$ implies $\lbconst[\n-1](\btheta^g_\n) > \ubconst[\n-1](\btheta^g_\n) - \epsconst \geq \constrain(\btheta^g_\n) - \epsconst$. 

Define $\bhtheta\coloneqq \argmax_{\btheta \in \constSet[,\epsconst]{}} \constrain(\btheta)$ and using both the above derived inequalities, we get, 
\begin{align*}
    \ubconst[\n-1](\bhtheta) &\leq \ubconst[\n-1](\btheta^g_\n) \tag{$\constSet[,\epsconst]{} \coloneqq \tilReachOper[\epsilon]{\safeset_0} \subseteq \optiSet[]{\n-1}$}\\ 
    &< \lbconst[\n-1](\btheta^g_\n) + \epsconst \tag{$\sumMaxwidth[]{\n-1}(\btheta^g_\n) < \epsconst$} \leq  \constrain(\btheta^g_\n) + \epsilon \\
\implies \constrain(\bhtheta) &\leq  \constrain(\btheta^g_\n) + \epsilon
\end{align*}
This implies, $\constrain(\btheta^g_\n) \geq  \max_{\btheta\in \constSet[,\epsconst]{}} \constrain(\btheta) - \epsilon$.
\end{proof} \vspace{-0.4 em}

\vspace{-.3em}
Next, we prove our main theorem using the lemmas above, and additionally guarantee the satisfaction of state and input constraints of the closed-loop system throughout the process.
\vspace{-.8 em}
\begin{proof}[Proof of Theorem \ref{thm:convergence}]
\looseness -1 The initial seed (\cref{assump:safe_seed}) ensures \cref{eq:constraint}, which implies that \MPC is feasible and thus satisfies state and input constraints at $\n=0$. 
In \autotuner, we sample at $\btheta_{\n} \in \pessiSet[]{\n-1} \implies \constrain(\btheta_\n) \geq \tau$ (\cref{lem:sagempc}) under \cref{assump:q_RKHS}, which ensures feasibility of the resulting closed-loop system \cref{eq:MPCFormulation} $\forall \n \geq 1$.

Finite time convergence guarantees follows from \cref{lem:sagempc} which provides a sample complexity bound, i.e., $\exists \n \leq \n^\star$ under \cref{assump:q_RKHS,assump:sublinear} before which \autotuner will terminate. 

Since \autotuner samples only if $\sumMaxwidth[]{\n-1}(\btheta^g_\n) > \epsconst$ and \cref{lem:sagempc} implies $\exists \n \leq \n^\star$ under which uncertainty in the $\pessiSet[]{\n-1}$ is uniformly bounded by $\epsconst$. This implies $\sumMaxwidth[]{\n-1}(\btheta^g_\n) \leq \epsconst$ and using this, \cref{lem:near_optimality} establishes the resulting optimality of the \autotuner algorithm.
\end{proof} \vspace{-0.5 em}

Note that our proving strategy differs from that of \goose and SageMPC. We employ a sampling strategy motivated by SageMPC; however, because we operate in a discrete domain, we define our sets using \goose's definitions. We do this to relate the growth of the pessimistic set with respect to the initial seed of the agent as in \goose while being able to use more efficient sample complexity bounds from \cite{prajapat2024safe}. 

\vspace{-.1em}